\documentclass[12pt,sort&compress]{elsarticle}
\usepackage{lineno,hyperref}
\usepackage{graphicx}
\usepackage{subcaption}
\usepackage{tabulary}
\usepackage{multirow}
\usepackage[top=1in,bottom=1in,left=1in,right=1in]{geometry}
\usepackage[dvipsnames]{xcolor}
\usepackage[labelfont=bf,textfont=it,font=footnotesize]{caption}
\usepackage{setspace}
\usepackage{parskip}
\usepackage{mathpazo}
\usepackage{amsmath}
\usepackage{amssymb}
\usepackage{amsthm}
\usepackage{cleveref}
\usepackage{color}
\usepackage{algorithm}
\usepackage{algorithmic}
\usepackage{nicefrac}       
\usepackage{cancel}
\usepackage{diagbox}
\usepackage{nth}

\newtheorem{theorem}{Theorem}

\newtheorem{assumption}[theorem]{Assumption}

\newtheorem{definition}{Definition}

\newtheorem{remark}{Remark}

\newcommand\Ccancel[2][black]{\renewcommand\CancelColor{\color{#1}}\cancel{#2}}

\DeclareMathOperator*{\sgn}{sign}

\newcommand{\x}{\textbf{x}}

\newcommand{\bi}{\begin{itemize}}
\newcommand{\ei}{\end{itemize}}
\newcommand{\bd}{\begin{displaymath}}
\newcommand{\ed}{\end{displaymath}}
\newcommand{\be}{\begin{align*}}
\newcommand{\ee}{\end{align*}}

\begin{document}

\begin{frontmatter}

\title{A Fast Saddle-Point Dynamical System Approach to Robust Deep Learning\footnote{*This work was partially supported by NSF CAREER Grant (CNS$\#1845969$)}}

\author[myprimaryaddress]{Yasaman Esfandiari}
\author[myprimaryaddress]{Aditya Balu}
\author[myprimaryaddress]{Keivan Ebrahimi}
\author[mysecondaryaddress]{Umesh Vaidya}
\author[mythirdaddress]{Nicola Elia}
\author[myprimaryaddress]{Soumik Sarkar}

\address[myprimaryaddress]{Iowa State University}
\address[mysecondaryaddress]{Clemson University}
\address[mythirdaddress]{University of Minnesota}
\cortext[mycorrespondingauthor]{Corresponding author}
\ead{soumiks@iastate.edu}

\journal{Neural Networks}

\begin{abstract}
Recent focus on robustness to adversarial attacks for deep neural networks produced a large variety of algorithms for training robust models. Most of the effective algorithms involve solving the min-max optimization problem for training robust models (min step) under worst-case attacks (max step). However, they often suffer from high computational cost from running several inner maximization iterations (to find an optimal attack) inside every outer minimization iteration. Therefore, it becomes difficult to readily apply such algorithms for moderate to large size real world data sets.
To alleviate this, we explore the effectiveness of iterative descent-ascent algorithms where the maximization and minimization steps are executed in an alternate fashion to simultaneously obtain the worst-case attack and the corresponding robust model. Specifically, we propose a novel discrete-time dynamical system-based algorithm 
that aims to find the saddle point of a min-max optimization problem in the presence of uncertainties. Under the assumptions that the cost function is convex and uncertainties enter concavely in the robust learning problem, we analytically show that our algorithm converges asymptotically to the robust optimal solution under a general adversarial budget constraints as induced by $\ell_p$ norm, for $1\leq p\leq \infty$. Based on our proposed analysis, we devise a fast robust training algorithm for deep neural networks. Although such training involves highly non-convex robust optimization problems, empirical results show that the algorithm can achieve significant robustness compared to other state-of-the-art robust models on benchmark data sets. 
\end{abstract}


\end{frontmatter}


\section{Introduction}
\label{Introduction}
The success of adversarial perturbations to input data for deep learning models poses a significant challenge for the machine learning community. The motivation for robustness against these adversaries stems from safety and life-critical applications of deep learning-based perception system such as self-driving cars, infrastructure assessment and security applications~\cite{sitawarin2018darts,biggio2013evasion, haghighat2020applications,hosseini2021prediction}. While pure white-box attacks~\cite{biggio2013evasion,carlini2017towards,goodfellow2014explaining,kurakin2016adversariala,szegedy2013intriguing,moosavi2016deepfool,yao2019trust} (where an adversary has full knowledge of the machine learning model) could be difficult to execute in practice, researchers have shown strong transferability of attacks~\cite{papernot2017practical} that can still cause significant damage. 

As attacks became more and more powerful, several defense strategies have also been proposed. A popular category of defense strategy is adversarial training, where adversarial examples are added to the training set followed by training the network using the augmented dataset~\cite{goodfellow2014explaining,szegedy2013intriguing}. However, such methods seem to be quite sensitive to adversarial budget used for generating the adversarial examples as well as other training hyper-parameters. A more powerful and stable defense mechanism originates from the min-max robust optimization problem~(${\cal RO}$)~\cite{bertsimas2011theory,ben2009robust,madry2017towards} by decoupling the minimization and maximization parts using the Danskin's theorem~\cite{danskin1966theory}. Here, the inner maximization refers to finding the adversarial perturbation that would maximize the training loss. On the other hand, the outer minimization deals with minimization of the training loss for the perturbed inputs. The decoupling process leads to the class of algorithms where, at a training epoch, one can find the worst-case attacks concerning the current model. Then a model parameter update step is executed following the traditional training process using perturbed training set with the worst-case attacks. However, finding the worst-case perturbation for deep learning models is quite non-trivial and cannot be guaranteed primarily due to the highly non-convex nature of the cost surface and non-concave nature of adversarial perturbations. Typically, powerful attacks such as projected gradient descent (PGD)~\cite{madry2017towards}, trade off-inspired adversarial defense via surrogate-loss minimization (TRADES)~\cite{DBLP:journals/corr/abs-1901-08573}, Free adversarial training~\cite{shafahi2019adversarial}, etc. are run in order to find the worst-case perturbations at every training epoch. However, it is observed empirically that the attacks with higher computational budgets seem to be more successful in approximating the worst-case perturbations, e.g., 20 step PGD is much stronger than a single step PGD. Therefore, it is usually quite expensive computationally to find a robust deep learning model with this algorithmic principle. Besides, there still remains a significant gap in the literature, in crafting theoretically sound algorithms for robust learning. There are only a few studies exploring the robust learning problem analytically. \citet{fawzi2016robustness} first analyzes robustness of nonlinear classifiers in a general noise regime, \citet{shaham2018understanding} proposed a framework to justify the performance of adversarial training theoretically, \citet{DBLP:journals/corr/abs-1901-08573} considered the trade off between robustness and accuracy, and the recent paper by \citet{lin2019gradient} proposes gradient descent ascent approach to solve the robust learning problem. Additionally,~\citet{sinha2017certifying} proposed an algorithm for solving the adversarial training problem from a distributionally robust optimization lens for finding the worst-case adversary based on the Lagrangian formulation of the objective function.

In this paper, we introduce a new algorithm for adversarial training of deep neural networks, which is motivated by the ${\cal RO}$ problem. Traditionally, robust optimization problems are solved iteratively by alternative steps of ascent and descent for achieving the worst case attack and finding the corresponding robust model for the attack. Our proposed defense approach called the stochastic saddle point dynamical systems (SSDS) algorithm, follows this principle to devise a fast algorithm for robust learning. 
We first formulate the min-max ${\cal RO}$ problem that arises in adversarial training of deep neural networks (DNNs) and show analytically that under strictly-convex and strictly-concave assumption for the loss function and the adversarial perturbations in the robust learning problem, SSDS algorithm converges asymptotically to the robust optimal solution, the min-max saddle point, under a general adversarial budget constraints as induced by $\ell_p$ norm, for $1\leq p\leq \infty$. 

Unlike existing approaches, our proposed algorithm does not decouple the minimization and maximization problems involved in robust optimization. Instead, it attempts to solve both problems simultaneously by evolving both the model parameters and the adversarial perturbations through the training epochs. As we do not attempt to find the worst-case perturbations at every training epoch, we save significant computation overhead as compared to other methods such as PGD-training~\cite{madry2017towards}, TRADES ~\cite{DBLP:journals/corr/abs-1901-08573}, and YOPO~\cite{zhang2019you}. While there are recent efforts to mitigate the computational overhead by using random projections ~\cite{wong2018scaling} instead of finding the worst case attacks, our approach is fundamentally different as we still try to solve the coupled robust optimization problem (that includes finding the worst case perturbations) while reducing the computations overhead. 

\textbf{Contributions:} Main contributions of this work are: (i) We propose a new Stochastic Saddle-point Dynamical Systems (SSDS) algorithm for robust learning that ensure satisfaction of the attack budget (induced by $\ell_p$ norm, for $1\leq p\leq \infty$) asymptotically without hard projections; (ii) We analyze the convergence characteristics for the proposed SSDS algorithm under the strictly-convex assumption for the loss function and strictly-concave assumption for the adversarial perturbations in the robust learning problem; (iii) We propose a mini-batch variant of the SSDS algorithm to make it feasible for deep learning problems; (iv) Finally, we provide empirical results (based on CIFAR-10~\cite{krizhevsky2014cifar} and other benchmark datasets) to show that the proposed approach is a computationally inexpensive way to train robust models for white- and black-box attacks in comparison with the state-of-the-art PGD and TRADES approaches. We also provide comparison with the stochastic gradient descent ascent algorithm (SGDA)~\cite{lin2019gradient} that can be treated as a baseline for our proposed approach.
\section{Related Work}
\label{RelatedWork}

Due to space constraints and a large amount of recent progress made by the adversarial machine learning community, our discussion of related work is necessarily brief. Here, we attempt to discuss the most recent \& relevant literature. We divide the section as: (1) adversarial attack/defense and (2) robust optimization.

\textbf{Adversarial Attack and Defense:} Initial evidence for vulnerability of deep classifiers to imperceptible adversarial perturbations was shown by \citet{szegedy2013intriguing}. Around the same time, \citet{biggio2013evasion} showed that SVMs could malfunction in security-sensitive applications and proposed a regularization term in the classifiers. In the deep learning community, \citet{goodfellow2014explaining} and \citet{kurakin2016adversariala} proposed the Fast Gradient Sign Method (FGSM) and its iterative variants as powerful attack strategies to fool deep learning models. adopting a similar notion, PGD~\cite{madry2017towards} attacks are crafted by going through iterative steps of generating attacks (e.g. $7$ and $10$ steps) to achieve more powerful adversaries. As the computational expenses associated with such attacks are significant, researchers have recently suggested different methods to mitigate that. \citet{shafahi2019adversarial} suggested using the gradients which is calculated for updating the model parameters to save computations and showed comparable results with state-of-the-art. \citet{wong2020fast} showed that comparably accurate models can be achieved in less time using FGSM along with several techniques from DAWNBench submissions~\citep{coleman2017dawnbench} (e.g., random initialization) to improve the performance. Additionally,~\citet{zhang2019you} showed that by restricting most of the forward and back propagation
within the first layer of the network during adversary updates, the computational time for achieving a robust model can be reduced. While these methods mainly focused on white box attacks, \citet{papernot2017practical} introduced the notion of black-box attacks where the adversary does not have complete knowledge of the learning model. Attacks can also be categorized into test-time~\cite{biggio2013evasion} and train-time (also known as data poisoning) attacks~\cite{khalid2018security}, and targeted and non-targeted attacks~\cite{barreno2010security}. In this paper, we only focus on test-time, non-targeted attacks. Apart from deep perception models such as Convolutional Neural Networks (CNN), researchers have also shown similar vulnerabilities of deep reinforcement learning (RL) models using similar philosophy~\cite{ijcai2017-525, havens2018online, tan2020robustifying, lee2020query, tan2020robust}. The notion of adversarial robustness is additionally useful in other areas of machine learning such as image synthesis~\cite{DBLP:journals/corr/abs-1906-09453,joshi2019semantic}.

Several defense approaches have been proposed in the literature, such as using denoising autoencoders-based Deep Contractive Networks~\cite{gu2014towards}, and defining a network robustness metric~\cite{NIPS2016_6339}. However, as discussed in the previous section, the most popular robust deep learning methods involve some form of adversarial training~\cite{2017arXiv170507204T,madry2017towards,goodfellow2014explaining,shaham2018understanding,DBLP:journals/corr/abs-1901-08573, zhang2019defense}. Defensive distillation~\cite{papernot2016distillation} is also another method of defense which showed fascinating results. However, \citet{carlini2017towards} could break such a defense mechanism by proposing multiple adversarial loss functions. \citet{athalye2018obfuscated} further analyzed various defense approaches and demonstrated that most existing defenses could be beaten by approximating gradients over defensively trained models. Recently, \citet{zhang2019defense} introduced Feature Scatter algorithm which increased the accuracy of robust models notably. This approach involves generating adversarial images through feature scattering in the latent space.

\textbf{Robust Optimization:}
\citet{shaham2018understanding} show that adversarial training of neural networks is, in fact, robustification of the network optimization, which we can exploit to increase the local stability of neural networks. Robust optimization has also been used in~\cite{Chen2017RobustOF} to find an approximately optimal min-max solution that optimizes for non-convex objectives. This method is based on a reduction from robust optimization to stochastic optimization. Here a $\alpha$-approximate stochastic oracle is given, and $\alpha$-approximate robust optimization in a convexified solution space is obtained. Nonetheless, ideas from robust optimization (closely related to regularization in machine learning~\cite{sra2011,xu2009}) for solving robust learning problems has not been explored sufficiently. Also, while there is recent work in robust optimization using continuous-time dynamical systems for the saddle point dynamics~\cite{ebrahimi2019}, it is deterministic, and its application to deep learning problems is not explored.

\section{Problem Formulation}
\label{Problem Formulation}
In this section, we first formally state the robust learning problem from a robust optimization point of view. Then we formulate the saddle-point dynamical system framework for solving robust learning problems. We consider a standard classification task under a data distribution ${\cal D}$ over the dataset $I = \{I^{1}, I^{2}, \cdots, I^{N}\}$, where, $I^{i} \in \mathbf{R}^m$ with set of labels, $y$. The loss function ({\it e.g.}, cross-entropy loss) is denoted by $L(I,y,w)$ with $w \in \mathbf{R}^{n}$ as the model parameters (decision variables). From a robust optimization (${\cal RO}$) perspective~\cite{bertsimas2011theory,ben2009robust,shaham2018understanding,ebrahimi2019}, robust learning can be written as
\begin{align} \label{RO}
{\cal RO}:=&\min_{w} \; \mathbb{E}_{(I,y) \sim {\cal D}} \; \Big[\underset{u \in {\cal U}}{\max}\;\; L(I+u,y,w)\Big]\;
\end{align}
where, the loss function $L$ is also a function of additive perturbation or uncertainty $u$ (constrained by uncertainty set $\cal U$) to the input. 

We approximate the expected loss using the standard empirical risk minimization for a finite number of i.i.d training samples, $I^{i} \ \text{for}\  i\in\{1,2,\cdots N\}$. We consider $u^{i}$ as the corresponding uncertainty for the data point $I^{i}$. Hence, ${\cal RO}$ problem (Eq.~\ref{RO}) can be written as
\begin{align}\label{ROdecomposable}
{\cal RO}:=&\min_{w} \textstyle \sum_{i=1}^N \underset{u^{i} \in {\cal U}^{i}}{\max}\; L(I^{i}+u^{i},y^{i},w)\;
\end{align}

The fundamental assumption in ${\cal RO}$ is that the uncertainty variables reside within the uncertainty sets
\begin{align}
{\cal U}^{i}:=\{u^{i}\in \mathbf{R}^{m}: h^{i}(u^{i})\leq 0\},\; i=1,\ldots,N\;,\label{budget_constraints}
\end{align}
where the $h^{i}$ functions representing the uncertainty sets are typically assumed to be convex functions such as norm-bound budgets. The goal here is to obtain model parameters, $w$, (e.g., weights and biases for neural networks) that are robust for all possible uncertainty parameter realizations within the uncertainty sets.

We assume that the ${\cal RO}$ problem (Eq. \ref{RO}) has at least one robust feasible solution.  Then we rewrite Eq.~\ref{RO} in an epigraph form~\cite{boyd2004}.
\begin{align} \label{RO_epigraph}
&\quad\quad\quad{\cal RO}:=\min_{w,t} \; t \nonumber\\
\text{s.t} & \sum_{i=1}^N \underset{u^{i} \in {\cal U}^{i}}{\max} L(I^{i}+u^{i},y^{i},w)-t \leq 0
\end{align}
This is an equivalent, albeit more convenient form for our framework, where $t$ is being added to the vector of model parameters as an auxiliary decision variable.

We define a Lagrangian multiplier $\lambda\geq0$ and the vector of model parameters as $x:=(w,t)$. We can write Eq.~\ref{RO_epigraph} as
\begin{align*}
\underset{x=(w,t)}{\min}\;\underset{\lambda \geq 0} \max \; \Big\{ t+\lambda \; \big(\sum_{i=1}^N \underset{u^{i} \in {\cal U}^{i}}{\max}\;L(I^{i}+u^{i},y^{i},w)-t\big)
\Big\}\;
\end{align*}   

Then the total Lagrangian can be written as
\begin{align}
{\cal L}(x,\lambda,u,v):=&\; t+\lambda \;\Big( \textstyle \sum_{i=1}^N \big(L(I^{i}+u^{i},y^{i},w) \nonumber\\ 
& \; -v^{i} \; h^{i}(u^{i}) \big)-t \Big) \;
\label{lag}
\end{align}
where the set $\{v^{i}\}$ are the Lagrangian multipliers for the inner 
maximization problem. Typically, loss function in an ${\cal RO}$ framework is considered to be a function of just the model parameters, not the uncertainty. Therefore, we formulate the algorithm with Lagrangian multipliers that ensure satisfaction of the attack budget asymptotically without hard projections. 
\begin{remark}
The ability to solve ${\cal RO}$ with general budget constraints without involving hard projection is advantageous from a computational perspective, as imposing hard projection may involve solving an optimization problem in itself.
\end{remark}

Derivation of the Lagrangian function along with the definition and properties of the saddle and KKT point of ${\cal RO}$ problem can be found in the supplementary material.

\section{Stochastic Saddle-Point Dynamical System Algorithm} \label{sec4}
Based on the problem setup laid out in the previous section, we now propose the stochastic saddle-point dynamical system (SSDS) algorithm for robust learning. We then introduce the mini-batch variant of the SSDS algorithm conducive to deep learning. Additionally, we also state the stochastic gradient descent ascent (SGDA) algorithm~(\citet{lin2019gradient}), which can be considered as a simplified baseline for the proposed SSDS algorithm.


\subsection{Algorithm Formulation}\label{SSDS}

In the previous section, we provided the total Lagrangian for obtaining the saddle and KKT point of ${\cal RO}$ problem. While that is sufficient to understand the saddle point dynamics of the ${\cal RO}$ problem, the implementation for deep learning problems necessitates the formulation of a stochastic discrete-time saddle point dynamical system to find the saddle point of the Lagrangian function. Based on the epigraph form of the optimization problem (Eq.~\ref{RO_epigraph}), we define $x:=(t,w),\ f(x):=t,\ g(x,u,\xi):=L(I^\xi+u^\xi,y,w)-t$, where $\xi\in \{1,\ldots,N\}$ is a random variable modeling the process for randomly selecting a data point out of $N$ possible samples. 

With these notations, we propose the following update rules for the parameters:
\begin{align}
&x_{k+1}=x_k-\alpha_k (\partial_x f(x_k)+\lambda_k \partial_x g(x_k,u_k,\xi_k))\label{sspd1}\;,\\
&\lambda_{k+1}=\big[\lambda_k+\alpha_k \big(g (x_k,u_k,\xi_k)-\textstyle \sum_{i=1}^N v_k^{i} h^{i}(u^{i}_k)\big)\big]_+\label{sspd2}\;,\\
&u^{i}_{k+1}=u^{i}_k+\alpha_k (\partial_{u^{i}} g(x_k,u_k)-v_k^{i} \partial_{u^{i}} h^{i}(u^{i}_k))\label{sspd3}\;,\\
&v^{i}_{k+1}=[v_k^{i}+\alpha_k \lambda_k h^{i}(u^{i}_k)]_+\;\;i=1,\ldots,N.\label{sspd4}
\end{align}
where, $[\cdot]_{+}$ is positive projection, $\xi_k$ is assumed to be an independent and identically distributed random process and $\alpha_k$ is the adaptive step-size with the following characteristics

\begin{align}\label{alpha}
    &\alpha_k=\frac{\gamma_k}{\Vert T(z_k)\Vert_{2}}\;,{\rm with} \;\;\gamma_k>0\;,\nonumber\\
&\sum_{k=1}^{\infty}\gamma_k=\infty\;,\; \textstyle \sum_{k=1}^{\infty} \gamma_k^{2}<\infty\;.
\end{align}
The following assumptions are made on $f$, $g$ and $h^{i}$ s. Note that the weights $x$ is updated using the loss function information for a randomly selected sample $I^\xi$. However, the uncertainty variables $u^i$ and associated multipliers $v^i$ are updated corresponding to all samples.   
\begin{assumption}\label{assume_ssds}
We assume that 
$f(x)$ is  convex in $x$ and each $h^{i}(u^{i})$ is convex in $u^{i}$.
Moreover, $g(x,u,\xi)$ is convex in $x$ and is strictly concave in $u$ for any fixed value of \;$\xi$.  
\end{assumption}

\begin{remark}
Clearly, the assumptions of strict convexity of $f$ and concavity of $g$ are not satisfied in the DNN setting. However, the strict convexity and strict concavity assumptions could be relaxed with weaker convergence results than the one reported below.

\end{remark}

The following theorem is the main result for asymptotic convergence of the discrete-time saddle point algorithm with diminishing step-size. We show that the update rules in (\ref{sspd1})-(\ref{sspd4}) lead to convergence to the KKT point (equivalent to the saddle point as specified in the supplementary material) of the ${\cal RO}$ problem.

\begin{theorem}
Let Assumption~\ref{assume_ssds} hold and we also assume that $\lambda^\star>0$ where $\lambda^\star$ is the saddle point of the Lagrangian~(for $\lambda$ in Eq.\ref{lag}), then, following is true for the SSDS algorithm with adaptive step-size $\alpha_k$ satisfying Eq.~\ref{alpha}.
\begin{align}
\lim_{k\to \infty}\mathbb{E}_{\xi_0^k}[x_k]=x^\star,&
\lim_{k\to \infty}\mathbb{E}_{\xi_0^k}[u_k]=u^\star\;,\\
\text{where}\quad \xi_0^k&=\{\xi_0,\ldots,\xi_k\}.\nonumber
\vspace{-20pt}
\end{align}
\end{theorem}
\begin{proof}
The proof of the theorem is provided in the supplementary material.

\end{proof}
The basic idea behind the convergence proof relies on proving the existence of saddle point $(x^\star, \lambda^\star, v^\star, u^\star)$ for the Lagrangian function, $\cal L$ 
\begin{eqnarray}
{\cal L}(x,\lambda,u,v):=f(x)+\lambda \\\left(g(x,u,\xi)-\sum_{i=1}^N v^i h^i(u^i)\right)
\end{eqnarray}
satisfying the following inequalities for all realization of the random variable $\xi$

${\cal L}(x^\star, \lambda, v^\star, u,\xi)\leq {\cal L}(x^\star, \lambda^\star, v^\star, u^\star,\xi)\leq {\cal L}(x, \lambda^\star, v, u^\star,\xi)$

It is important to emphasize that although the function $f$, $g$, and $h$ are assumed to be convex with respect to $x$ and concave with respect to $u$ (Assumption \ref{assume_ssds}), the Lagrangian function, ${\cal L}$, is not jointly concave  with respect to $(u,\lambda)$. The lack of joint concavity makes the saddle point proof nontrivial and different from the saddle point problem for a general min-max optimization problem~(\citet{lin2019gradient}). The existence of saddle point for min-max optimization problem under convexity-concavity assumption is a standard result in convex optimization and is at the heart of the convergence proof of the stochastic gradient descent-ascent algorithm.

\begin{remark}
Although the stability is not shown in this paper, we observe in practice that the dynamics without $\lambda$ in the $v$-update (Eq.\ref{sspd4}) works for both active and inactive constraints (whether $\lambda^\star$ is positive or zero) of the ${\cal RO}$ problem and converges to the KKT point. Therefore, $\lambda$ can be removed from $v$-update in practice for solving the ${\cal RO}$ problem.
\end{remark}

\begin{algorithm}[ht]
\caption{Mini-batch SSDS algorithm}
\begin{algorithmic}[1]\label{ssds}
     \STATE \textbf{Input}: $\varepsilon$, \textit{lr}, $p$, $C_1$, $C_2$
     \STATE \textbf{Initialization}: $\lambda_0$, $\alpha_0$,$w_0$, $t_0$,$u_0$, $v_0$ 
    \FOR{$k\in\{1,...,K\}$}
    \STATE distribute mini-batches as $m=\{m_0,m_1,..,m_n\}$
    \STATE $w_k^{m_0}=w_k$
    \STATE $\lambda_k^{m_0}=\lambda_k$
        \FOR{$m_j\in\{m_0,m_1,...,m_n\}$} 
        \STATE $ \partial_{w_k} = \partial_{w_k}\sum_{j \in m_j}  L(I^{j}+u_k^{j},y^{j},w_k^{j})$
        \STATE $w_k^{m_{j+1}}=w_k^{m_{j}}-lr  \;\lambda_k(\partial_{w_k})$ 
        \ENDFOR
    \STATE $t_{k+1}=t_k+\alpha_k(\lambda_k-1)$
    \STATE $v_{k+1}^{m_j}=v_k^{m_j}+\alpha_k(\Vert u_k^{m_j}\Vert_\infty-\varepsilon)$
    \STATE $\partial_{u_k} = \partial_{u_k} L(I^{m_j}+u_k^{m_j},y^{m_j},w_k^{m_j})$
    \STATE $u_{k+1}^{m_j}=u_{k}^{m_j}+\alpha_k(\partial_{u_k}-C_1 \; v_k^{m_j}\sgn(u_k^{m_j}))$
    \FOR {$j \in m_j$}
        \STATE $B^{j} =  (\Vert u_k^{j}\Vert_\infty-\varepsilon)$
        \STATE $U^{j} = \big(L(I^{j}+u_k^{j},y^{j},w_k^{j})-v_k^{j}\big)B^{j}$
    \ENDFOR
    
    \STATE $U = \sum_{j \in m_j} U^{(j)}$
    \STATE $\lambda_{k+1}^{m_j}=\lambda_{k}^{m_j}+C_2 \;
    \alpha_k \Big(U -t_k\Big)$
    \STATE $w_k=w_k^{m_n}$
    \STATE $\lambda_k=\lambda_{k+1}^{m_n}$
    \STATE $\alpha_{k+1}=\alpha_k e^{-k\textit{p}}$
    \ENDFOR
\end{algorithmic}
\end{algorithm}

\subsection {Mini-batch Implementation of SSDS Algorithm}\label{ssdsmini}

In an attempt to use the proposed approach for robust training of DNNs, we propose a mini-batch variant of the SSDS algorithm to achieve a more stable convergence. As stated above, the SSDS algorithm involves the decision variable $x:=(t,w)$, where the set $\{w\}$ is the parameters of DNN. For simplicity of implementation, we first separate the update rule for $x$, described in Eq.~\ref{sspd1}. In other words, we split the updates of $w$ and $t$ that also enables us to use standard learning rates (denoted by \textit{lr}) for the $w$ updates. For the updates of $t$ and other SSDS variables, such as $\lambda$, $u$ and $v$, we use a diminishing step-size $\alpha_k$. However, we refrain from applying the diminishing step-size described in (Eq.~\ref{alpha}) due to the sheer complexity involved in taking the norm of the parameters for a large-scale neural network. Instead, we use an exponentially decaying diminishing step-size, $\alpha_{k+1}=\alpha_k e^{-k\textit{p}}$, where $p$ is the decay rate for exponentially diminishing step-size and $k$ is the epoch number. Note that the updates of $u$ and $\lambda$ can experience scaling issues depending on the values of the gradients and variable $v_k$. Therefore, we add two scaling factors $C_1$ and $C_2$ in the update rules of $u$ and $\lambda$ to bring different terms of the update laws to the same scale. For a given data set and model architecture, we find appropriate values by a few trial and error steps. Due to the separation of the updates, another small departure in our implementation from the prescribed algorithm is - while we perform $w$ updates for every mini-batch ($m_j$ refers to the $j^{th}$ mini-batch of $k^{th}$ epoch), we update other parameters once every epoch. Also, in the original formulation, we continuously update $u$ corresponding to all the images while the $w$ is updated using the gradient information of the loss function evaluated at randomly selected images. In the algorithm implementation, however, $u$ is also updated only corresponding to the randomly selected images based on which the network weights are updated. This approach helps in reducing computation for large training sets. 

Based on the above setup, we present the mini-batch SSDS algorithm (Algorithm~\ref{ssds}). We can also craft attacks using this. To do that, we run iterative updates of the perturbations $u$ given a test sample along with its corresponding Lagrangian multiplier $v$, keeping the model $\{w^\star\}$ and $\lambda^*$ fixed. We discuss the attack algorithm in the supplementary material. 

\subsection{Stochastic Gradient Descent Ascent (SGDA)}\label{sec3:SGDA}
In order to study the effectiveness of the Lagrangian formulation to solve the robust optimization problem, we consider the baseline SGDA algorithm~(\citet{lin2019gradient}), which is essentially a simplification of the proposed SSDS algorithm as discussed below. We provide a mini-batch variant of the SGDA algorithm in Algorithm~\ref{SGDA} to make it applicable in the deep learning setting. Comparing these two algorithms, we can readily see that algorithm~\ref{SGDA} is a special case of algorithm~\ref{ssds} where the Lagrangian multipliers ($\lambda$ and $v$) are removed. Note, an important analytical advantage of the robust optimization-based formulation involving the Lagrangian function (Eq.~\ref{lag}) is that it allows us to incorporate general convex budget constraints (Eq.~\ref{budget_constraints}) on the adversary without involving complicated projection operations. In particular, with the Lagrangian multipliers $\lambda$ and $v^{i}$, we can ensure that the budget constraints are met asymptotically. It is important to emphasize that the use of projection operator to impose general convex budget constraints involves solving an optimization problem in itself and hence it could considerably increase the computational time. However, when the budget constraints are simple such as $\ell_\infty$ norm constraints, the projection operation is straight forward. For such cases, gradient descent ascent algorithm can be utilized for adversarial training without involving the multipliers.

\begin{algorithm}
\caption{Mini-batch SGDA algorithm}
\begin{algorithmic}[1]\label{SGDA}
     \STATE \textbf{Input}: $\varepsilon$, \textit{lr}, $p$
     \STATE \textbf{Initialization}: $\alpha_0$,$w_0$,$u_0$
    \FOR{$k\in\{1,...,K\}$}
    \STATE distribute mini-batches as $m=\{m_0,m_1,..,m_n\}$
    \STATE $w_k^{(m_0)}=w_k$
        \FOR{$m_j\in\{m_0,m_1,...,m_n\}$} 
        \STATE $ \partial_{w_k} = \partial_{w_k}\sum_{j \in m_j}  L(I^{j}+u_k^{j},y^{j},w_k^{j})$
        \STATE $w_k^{(m_{j+1})}=w_k^{(m_{j})}-lr  \;(\partial_{w_k})$ 
        \ENDFOR
    \STATE $u_{k+1}^{m_j}=u_{k}^{m_j}+\alpha_k(\partial_{u_k} L(I^{m_j}+u_k^{m_j},y^{m_j},w_k^{m_j})$
    \STATE $w_k=w_k^{m_n}$
    \STATE $\alpha_{k+1}=\alpha_k e^{-k\textit{p}}$
    \ENDFOR
\end{algorithmic}
\end{algorithm}

\begin{figure*} [ht]
\centering
    \begin{subfigure}{0.49\textwidth}
    \centering
        \includegraphics[width=0.6\linewidth]{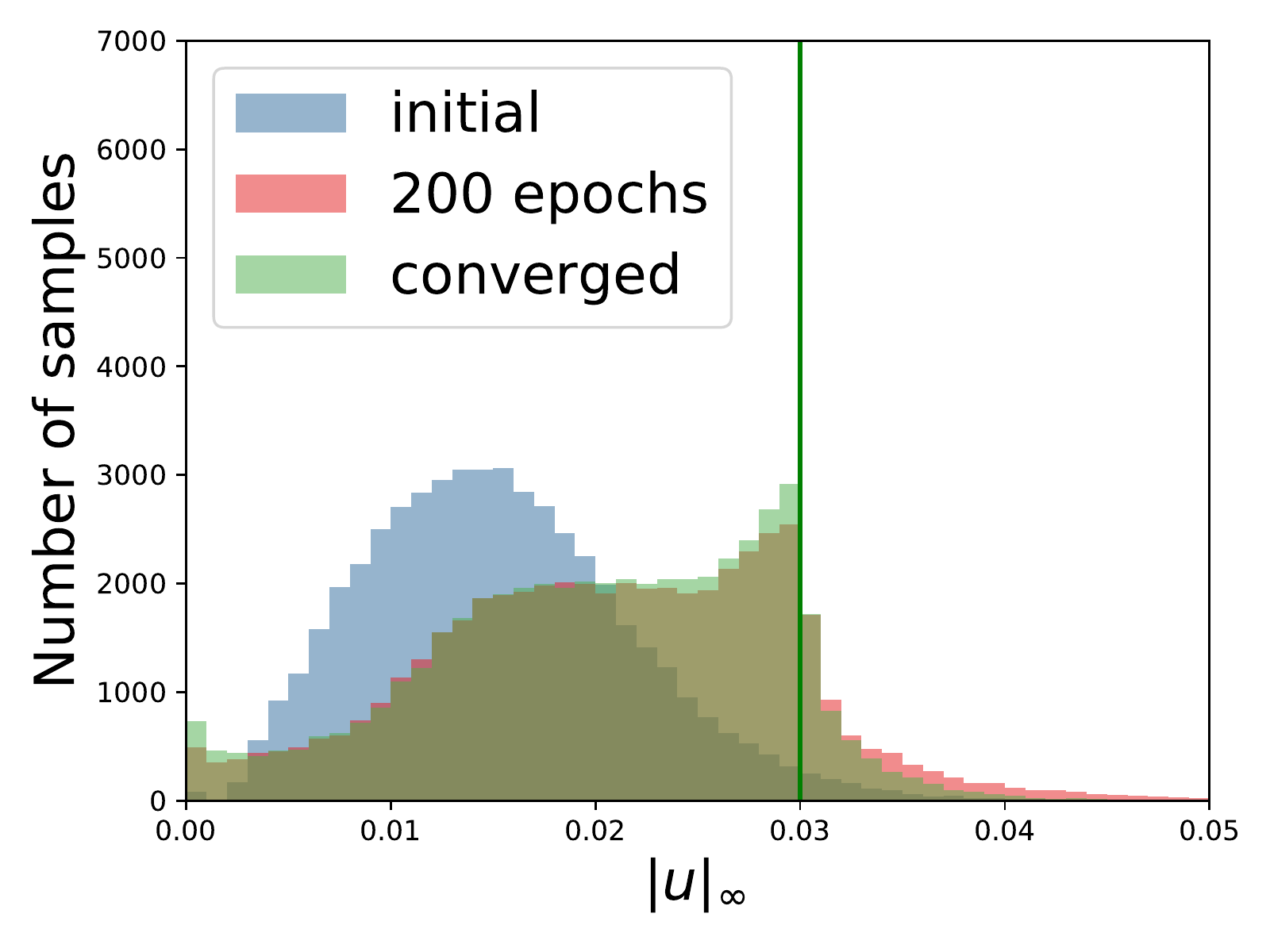}
        \caption{}
        \label{fig:ssdshist}
    \end{subfigure}
        \begin{subfigure}{0.49\textwidth}
    \centering
        \includegraphics[width=\linewidth,trim={0in, 1.5in, 0in, 1.5in}, clip]{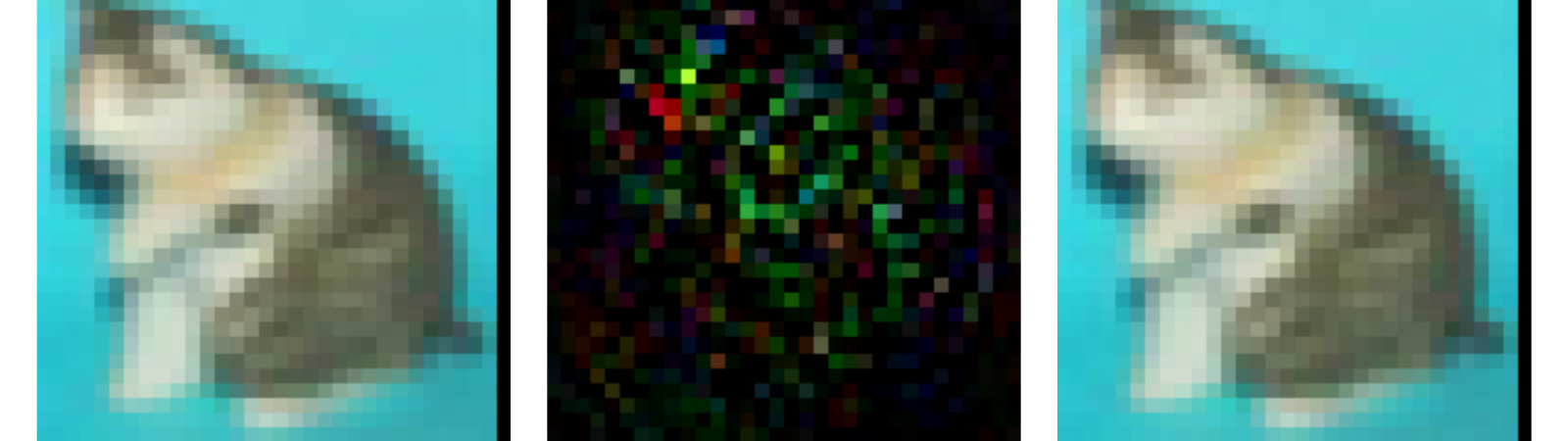}
        \caption{}
        \label{fig:ssdsimg}
    \end{subfigure}\\
    \hfill\begin{subfigure}{0.33\textwidth}
        \includegraphics[width=0.8\linewidth]{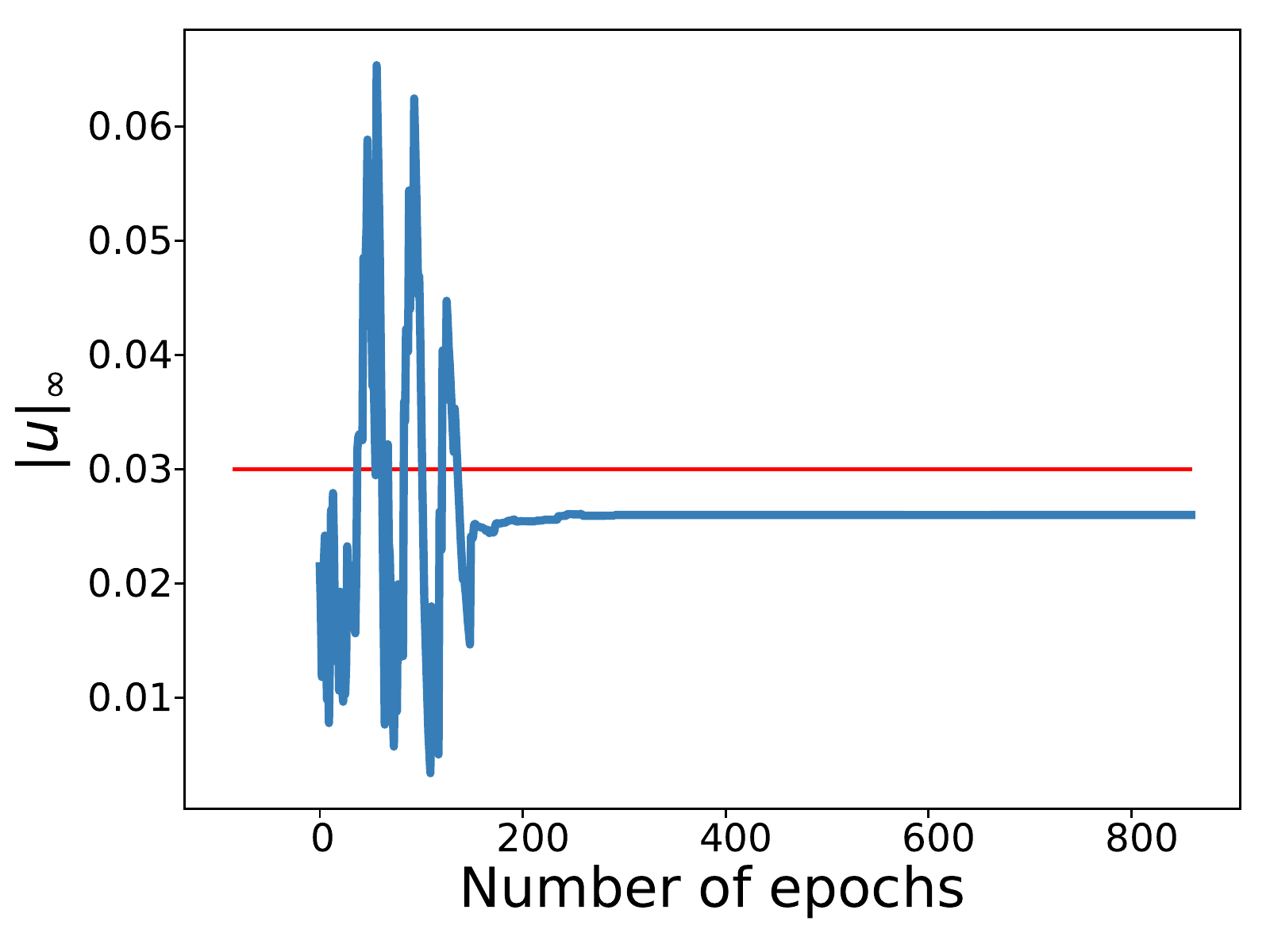}
        \caption{}
        \label{fig:uplot}
    \end{subfigure}\hfill
    \begin{subfigure}{0.33\textwidth}
        \includegraphics[width=0.8\linewidth]{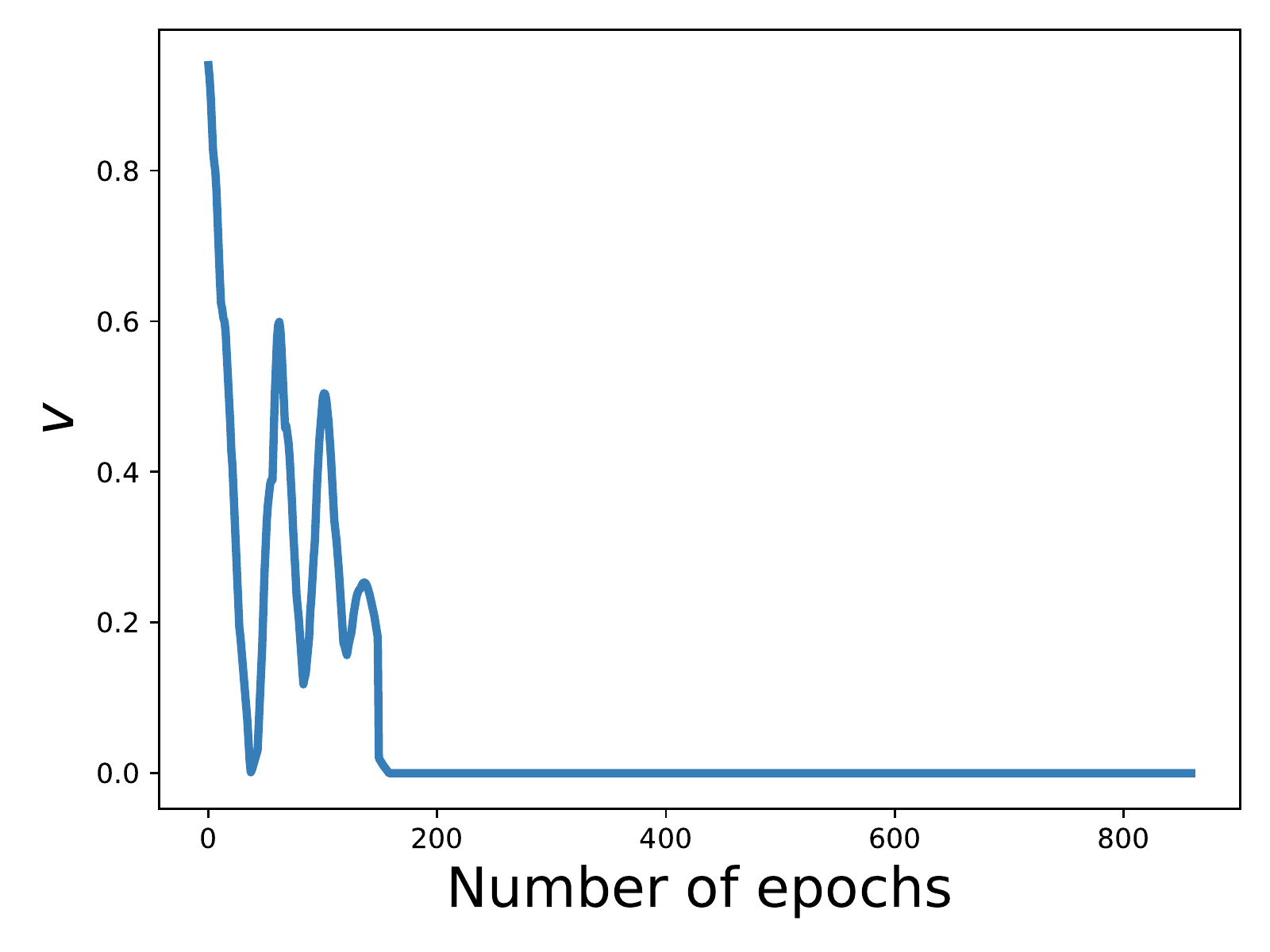}
        \caption{}
        \label{fig:vplot}
    \end{subfigure}\hfill
    \begin{subfigure}{0.33\textwidth}
        \includegraphics[width=0.8\linewidth]{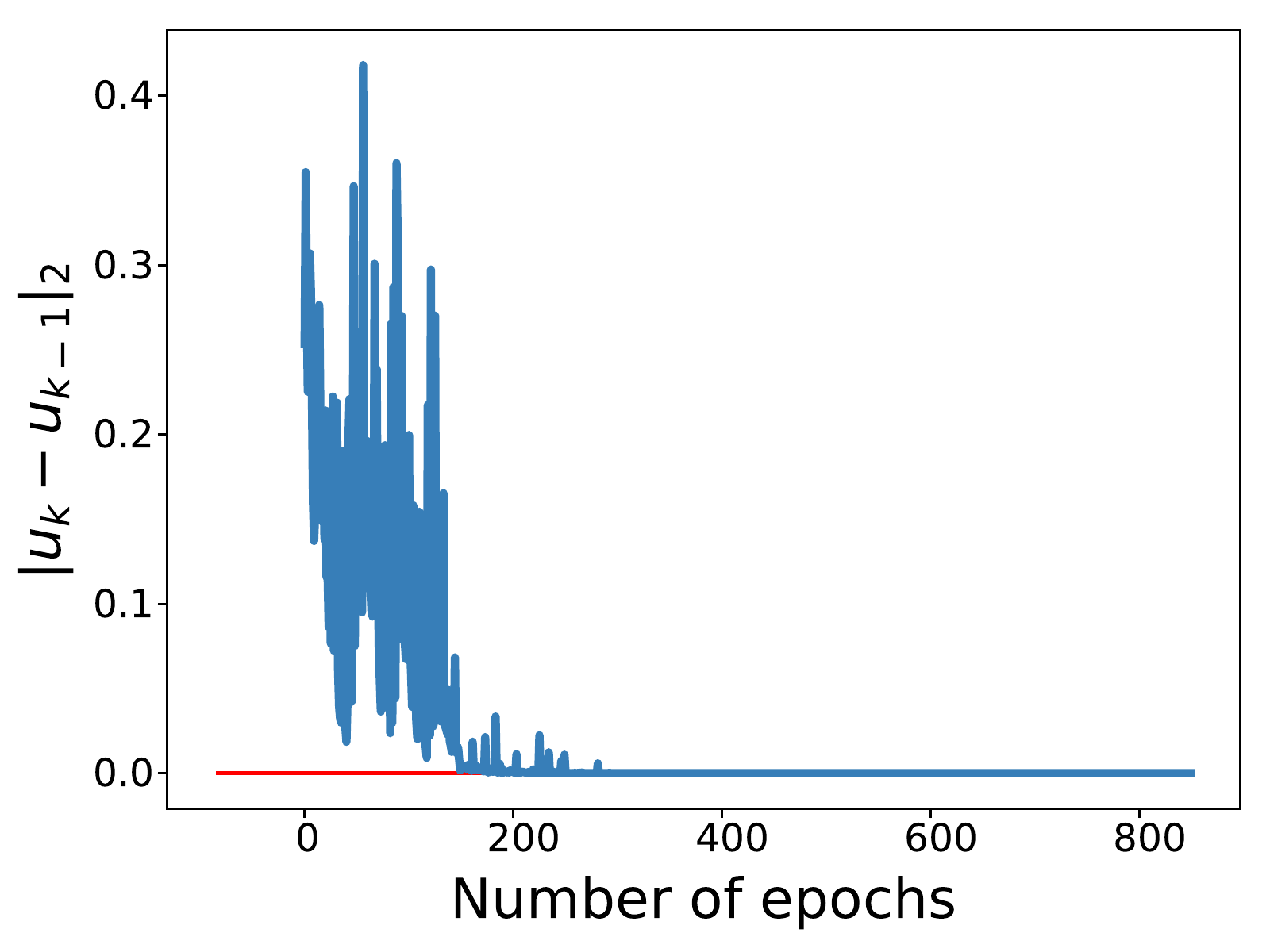}
        \caption{}
        \label{fig:difplot}
    \end{subfigure}
    \caption{\textit{Results on CIFAR-10 dataset using VGG19 model}: (a) $\Vert u\Vert_\infty$ histogram (b) Randomly chosen image for SSDS training (left), with its corruption(center) and the corrupted image(right) (c) Evolution of $\Vert u \Vert_\infty$ for the above image  (d) values of $v$ for the above image  (e) Evolution of $\Vert u_{k}-u_{k-1} \Vert_2$ for the above image}\label{ssdsvis}
\end{figure*}

\section{Experimental Results} \label{sec5}
In this section, we analyze the performance of SSDS algorithm empirically. We compare the effectiveness of our algorithms with the baseline SGDA algorithm and other state-of-the-art methods (e.g., PGD, TRADES, Free-m, YOPO,etc).

\textbf{Experiment Setup:} We present the empirical studies on the CIFAR-10, CIFAR-100 and SVHN datasets. We use VGG19~\cite{VGG19}, ResNet50~\cite{he2015deep}, and WideResNet~\cite{wideresnets} for classification.
The perturbation size is $\varepsilon=0.03$, and the exponential step-size decay parameter $p$ is set to be $0.001$. We initialize $\lambda=4,\;\nu=1.\;C_1=C_2=0.01,\;t=1$ for SSDS training. The perturbation step-size is set to be $\eta=2$ for SGDA and SSDS training. For the purpose of consistency with other research work in this area, we use $\eta=0.007$ ($\frac{2}{255}$) for crafting PGD and CW adversaries.
The code is available at our GitHub repository: \href{https://github.com/yasesf93/SSDS}{https://github.com/yasesf93/SSDS}.

\subsection{SSDS Convergence Characteristics}
We begin the discussion on SSDS algorithm characteristics by observing the behavior of the perturbations involved in SSDS training. We plot the histogram of $\ell_{\infty}$-norms of final perturbations ($\textit{u}$) added to the training images for a few epochs during the training process (see Fig.~\ref{fig:ssdshist}). This is to verify the theoretical claim that the final perturbations for the training images should converge at or below the budget. From the empirical results shown in Fig.~\ref{fig:ssdshist}, we make the observation that although perturbations for some images spill over the threshold value ($\varepsilon= 0.03$) during the course of the training process, most of the perturbations converge within the bound eventually (interestingly, a large number of the perturbations settle below the threshold). Next, we focus on a specific training sample to understand how the dynamics for different variables in the algorithm evolve during the course of the training process. We observe the dynamics of $v$ over the training epochs. In this experiment, $v$ was initialized at $1$ (for all of the images in the training set), and it converges to $0$ after around epoch $180$. Similarly, as shown in Fig.~\ref{fig:uplot}, $\ell_{\infty}$-norm of the perturbation generated for the chosen image converges to $0.026$ which is below the specified budget. However, the actual perturbations for individual pixels still continue to evolve even after the $\ell_{\infty}$-norm for the entire perturbation matrix settles down. To monitor the perturbations for the individual pixels, we plot $\ell_2$ norm of the difference between the perturbations for two consecutive epochs. We see that this metric finally converges to $0$ around epoch $300$. At this point, the overall training process also converges except for small changes due to the stochastic nature of the training algorithm. 

\begin{remark}
Analytically and experimentally, we show that the final perturbation should remain within the budget after convergence. However, for the purpose of consistency with other algorithms we compare against, we apply the projection term to the perturbations to ensure the perturbations never exceed the budget and we call this variant SSDS-p.
\end{remark} 

\begin{algorithm}
\caption{Mini-batch SSDS-p algorithm}
\begin{algorithmic}[1]\label{ssdsp}
     \STATE \textbf{Input}: $\varepsilon$, \textit{lr}, $p$, $C_1$, $C_2$
     \STATE \textbf{Initialization}: $\lambda_0$, $\alpha_0$,$w_0$, $t_0$,$u_0$, $v_0$ 
    \FOR{$k\in\{1,...,K\}$}
    \STATE distribute mini-batches as $m=\{m_0,m_1,..,m_n\}$
    \STATE $w_k^{m_0}=w_k$
    \STATE $\lambda_k^{m_0}=\lambda_k$
        \FOR{$m_j\in\{m_0,m_1,...,m_n\}$} 
        \STATE $ \partial_{w_k} = \partial_{w_k}\sum_{j \in m_j}  L(I^{j}+u_k^{j},y^{j},w_k^{j})$
        \STATE $w_k^{m_{j+1}}=w_k^{m_{j}}-lr  \;\lambda_k(\partial_{w_k})$ 
        \ENDFOR
    \STATE $t_{k+1}=t_k+\alpha_k(\lambda_k-1)$
    \STATE $v_{k+1}^{m_j}=v_k^{m_j}+\alpha_k(\Vert u_k^{m_j}\Vert_\infty-\varepsilon)$
    \STATE $\partial_{u_k} = \partial_{u_k} L(I^{m_j}+u_k^{m_j},y^{m_j},w_k^{m_j})$
    \STATE $u_{k+1}^{m_j}=u_{k}^{m_j}+\alpha_k(\partial_{u_k}-C_1 \; v_k^{m_j}\sgn(u_k^{m_j}))$
    \FOR {$j \in m_j$}
        \STATE $B^{j} =  (\Vert u_k^{j}\Vert_\infty-\varepsilon)$
        \STATE $U^{j} = \big(L(I^{j}+u_k^{j},y^{j},w_k^{j})-v_k^{j}\big)B^{j}$
    \ENDFOR
    
    \STATE $U = \sum_{j \in m_j} U^{(j)}$
    \STATE $\lambda_{k+1}^{m_j}=\lambda_{k}^{m_j}+C_2 \;
    \alpha_k \Big(U -t_k\Big)$
    \STATE $\lambda_{k+1}^{m_j} \leftarrow \Pi(\lambda_{k+1}^{m_j})$ \; ($\Pi$ is the projection operator)
    \STATE $w_k=w_k^{m_n}$
    \STATE $\lambda_k=\lambda_{k+1}^{m_n}$
    \STATE $\alpha_{k+1}=\alpha_k e^{-k\textit{p}}$
    \ENDFOR
\end{algorithmic}
\end{algorithm}

\subsection{SSDS vs SSDS-p Convergence Results}
As discussed in the main paper, we can show theoretically and experimentally that although the perturbations might spill over the threshold value ($\varepsilon=0.03$) for some epochs, they will converge at or below the threshold after the convergence. However, to be consistent with other research works in this area, we apply a projection term to ensure that the perturbations always remain within the budge during the training process which results in the SSDS-p which is presented in Algorithm~\ref{ssdsp}.

As Fig~\ref{ssdsptrfigs} shows, by applying SSDS-p training the $\ell_\infty$ norm of the perturbations ($\Vert u\Vert_\infty$) always stays withing the boundary (Fig~\ref{fig:ssdsphist}) whereas without the projection (Fig~\ref{fig:ssdshist2}) the perturbations may spill over the boundary.

\begin{figure*}[ht]
    \centering
    \begin{subfigure}{0.49\textwidth}
        \includegraphics[width=0.99\linewidth]{Figures/infdelnormhist.pdf}
        \caption{\label{fig:ssdshist2}}
    \end{subfigure}
    \begin{subfigure}{0.49\textwidth}
        \includegraphics[width=0.99\linewidth]{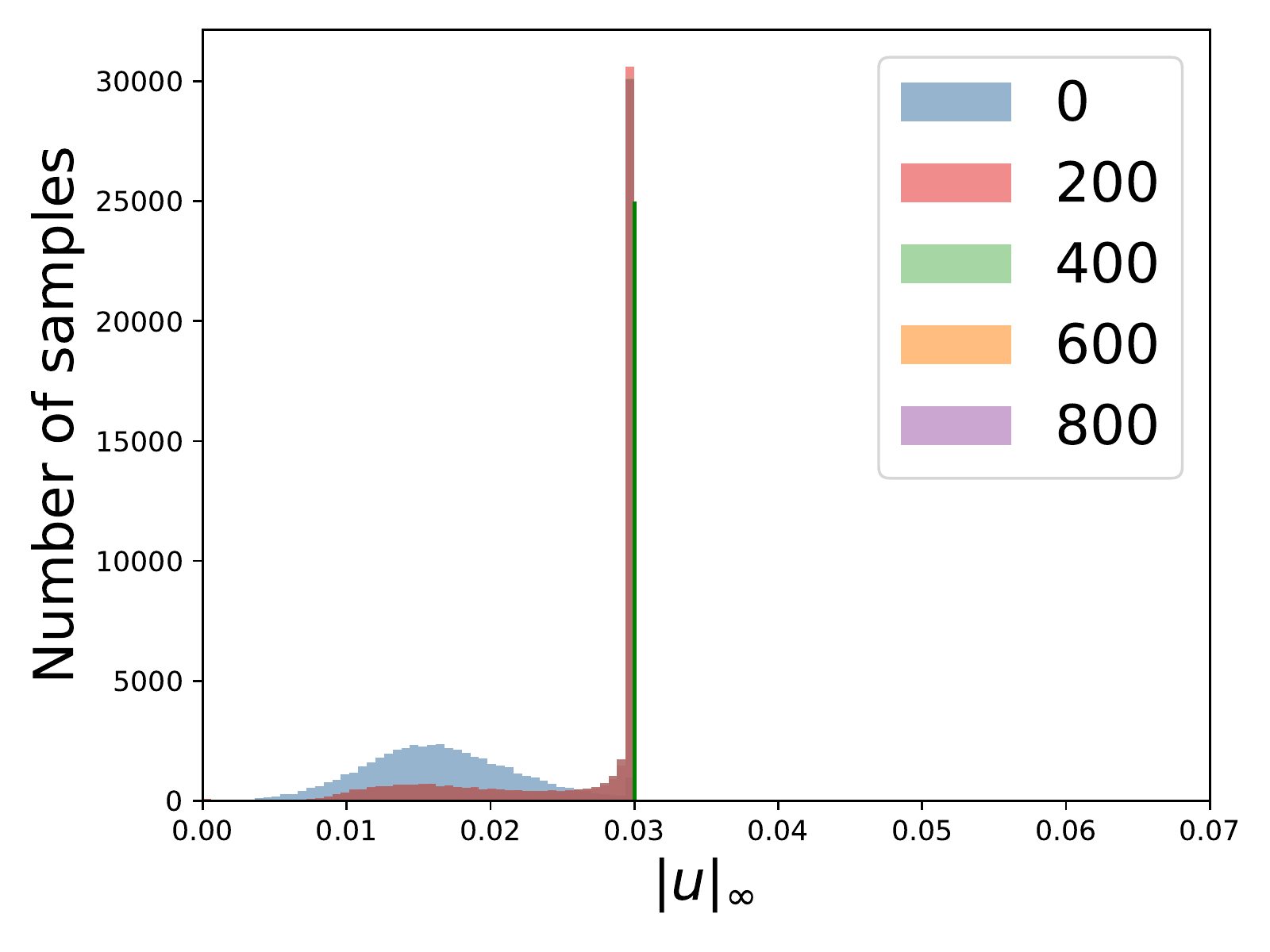}
        \caption{\label{fig:ssdsphist}}
    \end{subfigure}
    \caption{\textit{CIFAR-10 dataset trained using SSDS vs SSDS-p algorithm, VGG19 model architecture}: (a) Histogram of $\Vert u\Vert_\infty$ for SSDS (b) Histogram of $\Vert u\Vert_\infty$ for SSDS-p }\label{ssdsptrfigs}
\end{figure*}

\subsection{SSDS-p performance comparison with SGDA as defense methods}
Unlike the computationally expensive techniques such as PGD training, both SGDA and SSDS-p compute the optimal perturbations/attacks cumulatively over the training epochs. Therefore, we begin this comparison between SGDA and SSDS-p algorithms by comparing the training schemes associated with each algorithm. As figures ~\ref{fig:twoacc} and \ref{fig:twoloss} show, the training scheme is very similar for SGDA and SSDS-p. However, the perturbation evolution is quite different. Fig~\ref{deltafigs} shows the $\ell_\infty$ norm of the perturbations added to a randomly selected images. In the case of SSDS-p training, the perturbation evolves gradually through the training epochs. On the other hand, for the SGDA method, the perturbation rapidly reaches the budget, and the method does not seem to explore the perturbations with lower magnitudes as much as SSDS-p. We note that this is a typical observation for these two algorithms. 
Based on the results in Table~\ref{wboxcomp}, SSDS-p algorithm performs slightly better for both clean and adversarially perturbed (across all attack algorithms considered here) input data and the marginal improvement can be attributed to better exploration of the perturbation space by the SSDS-p algorithm. 

\begin{table} [ht]
    \caption{SGDA and SSDS-p defense model (WideResNet) performance comparison under white-box attacks}
    \begin{subtable}{0.49\linewidth}
        \centering
        \resizebox{0.5\textwidth}{!}{
        \begin{tabular}{c|c}
            Attack & Accuracy\\
            \hline
            Clean & 81.97\% \\
            FGSM & 80.4\% \\
            PGD & 44.32\% \\
            SGDA & 48.58\% \\
            SSDS-p & 51.64\% \\
        \end{tabular}
        }
      \caption{SGDA training}
      \label{tab:SGDA}
    \end{subtable}
    \hfill
    \begin{subtable}{0.49\linewidth}
        \centering
        \resizebox{0.5\textwidth}{!}{
      \begin{tabular}{c|c}
            Attack & Accuracy\\
            \hline
            Clean & 82.91 \% \\
            FGSM & 81.21 \% \\
            PGD & 45.89\% \\
            SGDA & 49.18 \% \\
            SSDS-p & 53.53 \% \\
        \end{tabular}
        }
        \caption{SSDS-p training}
        \label{tab:ssds}
     \end{subtable}
    \label{wboxcomp}
\end{table}

We then visualize the attack on randomly selected images. As Fig~\ref{viscomp} suggests, the two algorithms seem to generate similar attacks although the SSDS-p attacks (Fig~\ref{fig:imgssdsp}) looks more gradient based. As the Lagrangian multiplier corresponding to generating the attacks ($v$) controls the gradients in each epoch, the gradients play a more significant part in generation the attacks.

\begin{figure}[ht]
    \centering
    \begin{subfigure}{0.49\textwidth}
    \includegraphics[width=0.99\linewidth,trim={0in 0.05in 0in 0in},clip]{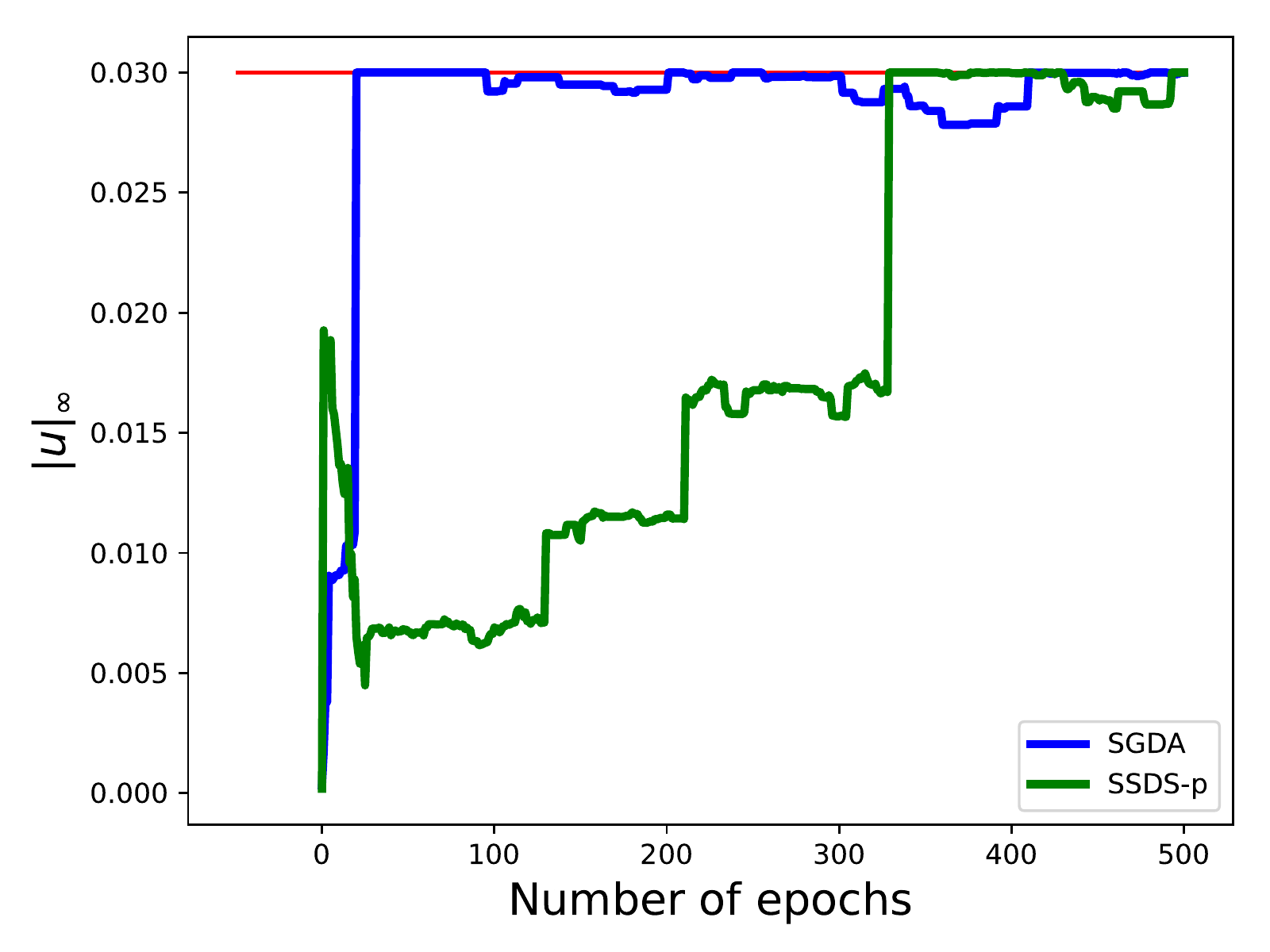}
    \caption{}
    \end{subfigure}
    \begin{subfigure}{0.49\textwidth}
        \includegraphics[width=0.99\linewidth]{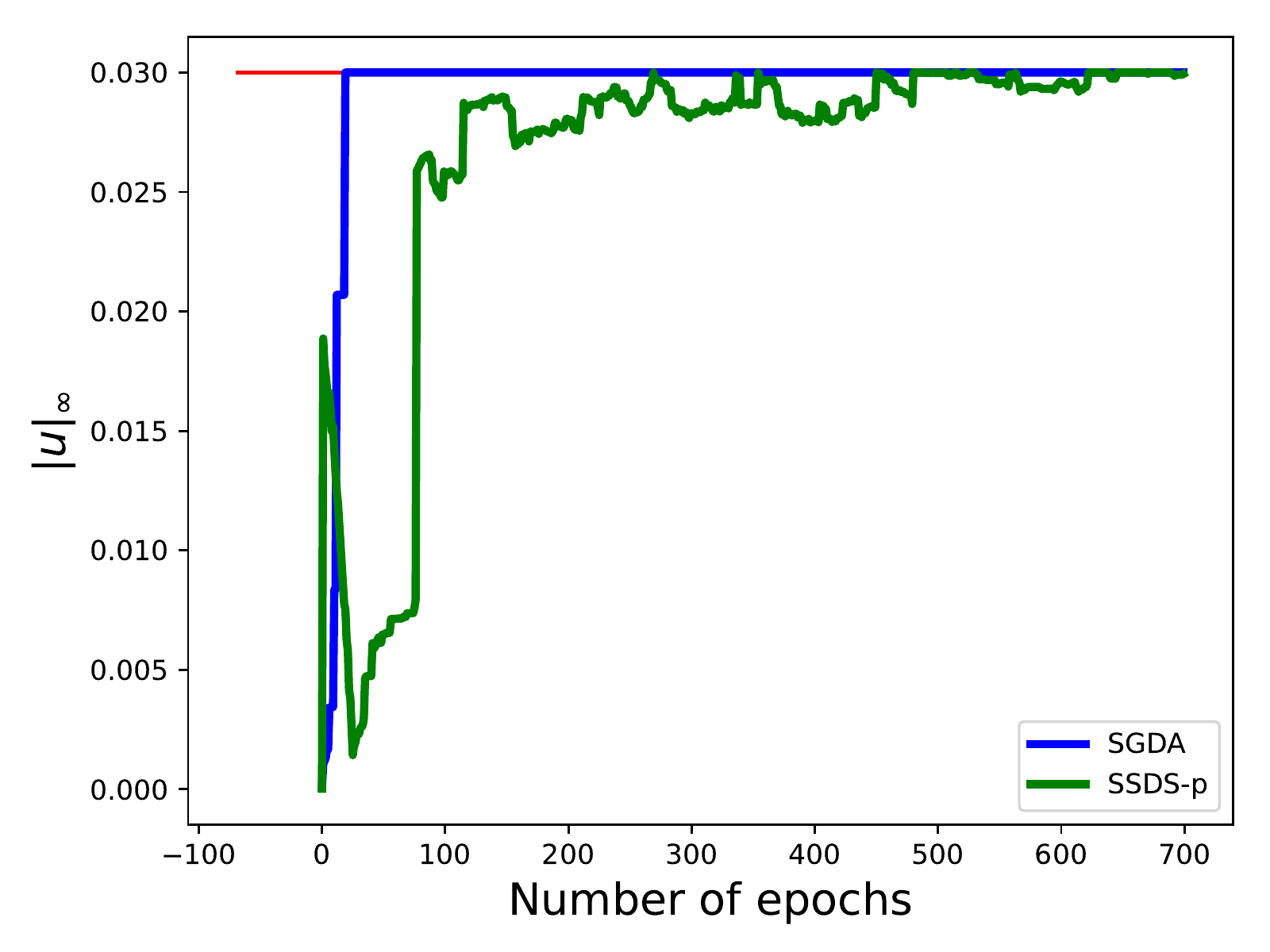}
        \caption{\label{fig:twodelta}}
    \end{subfigure}
    \caption{Evolution of perturbation/attack on two randomly selected images under SGDA and SSDS-p training}
    \label{deltafigs}
\end{figure}

\begin{figure*}[ht]
    \centering
    \begin{subfigure}{0.49\textwidth}
        \includegraphics[width=0.99\linewidth]{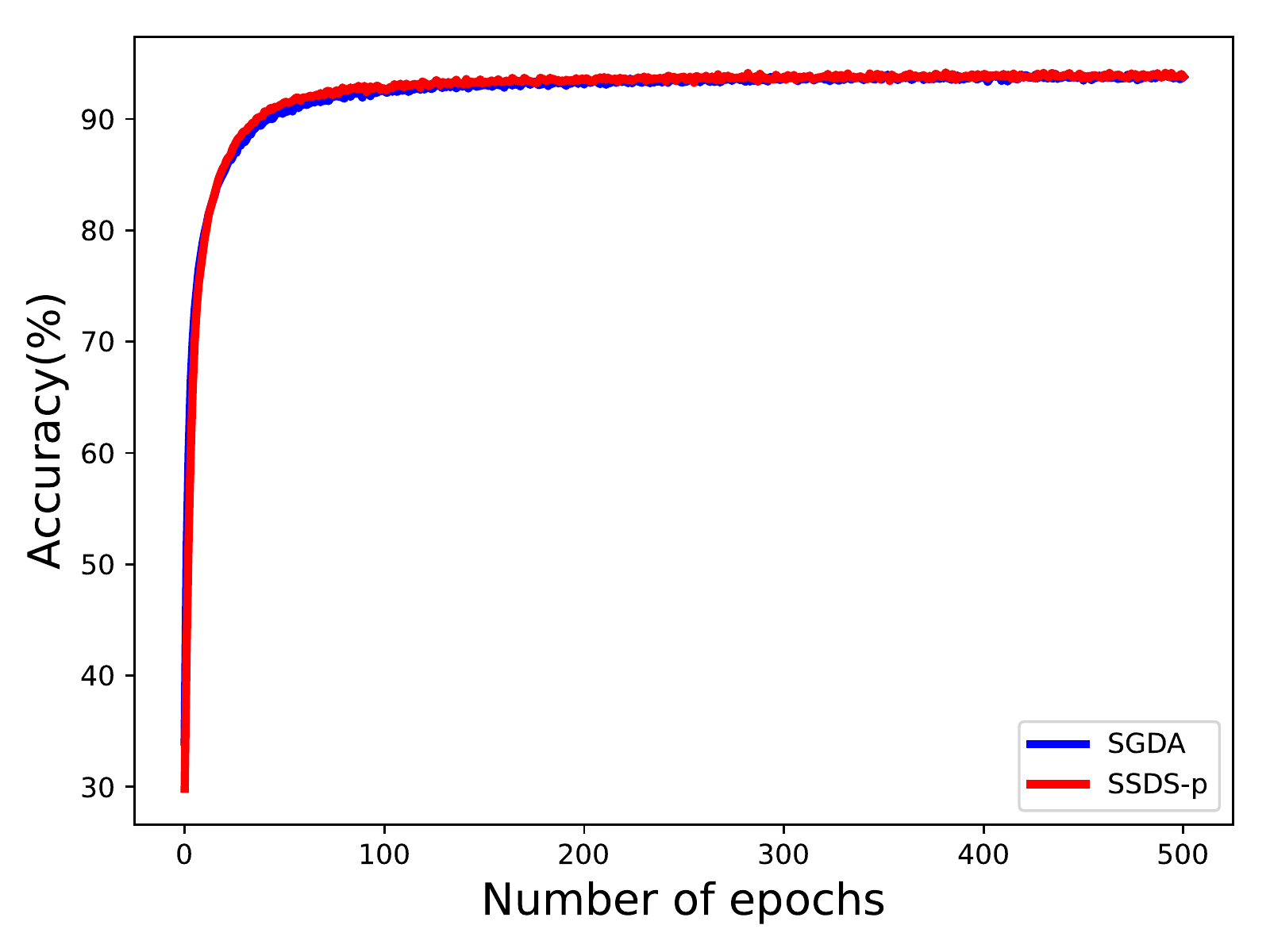}
        \caption{}
        \label{fig:twoacc}
    \end{subfigure}
    \begin{subfigure}{0.49\textwidth}
        \includegraphics[width=0.99\linewidth]{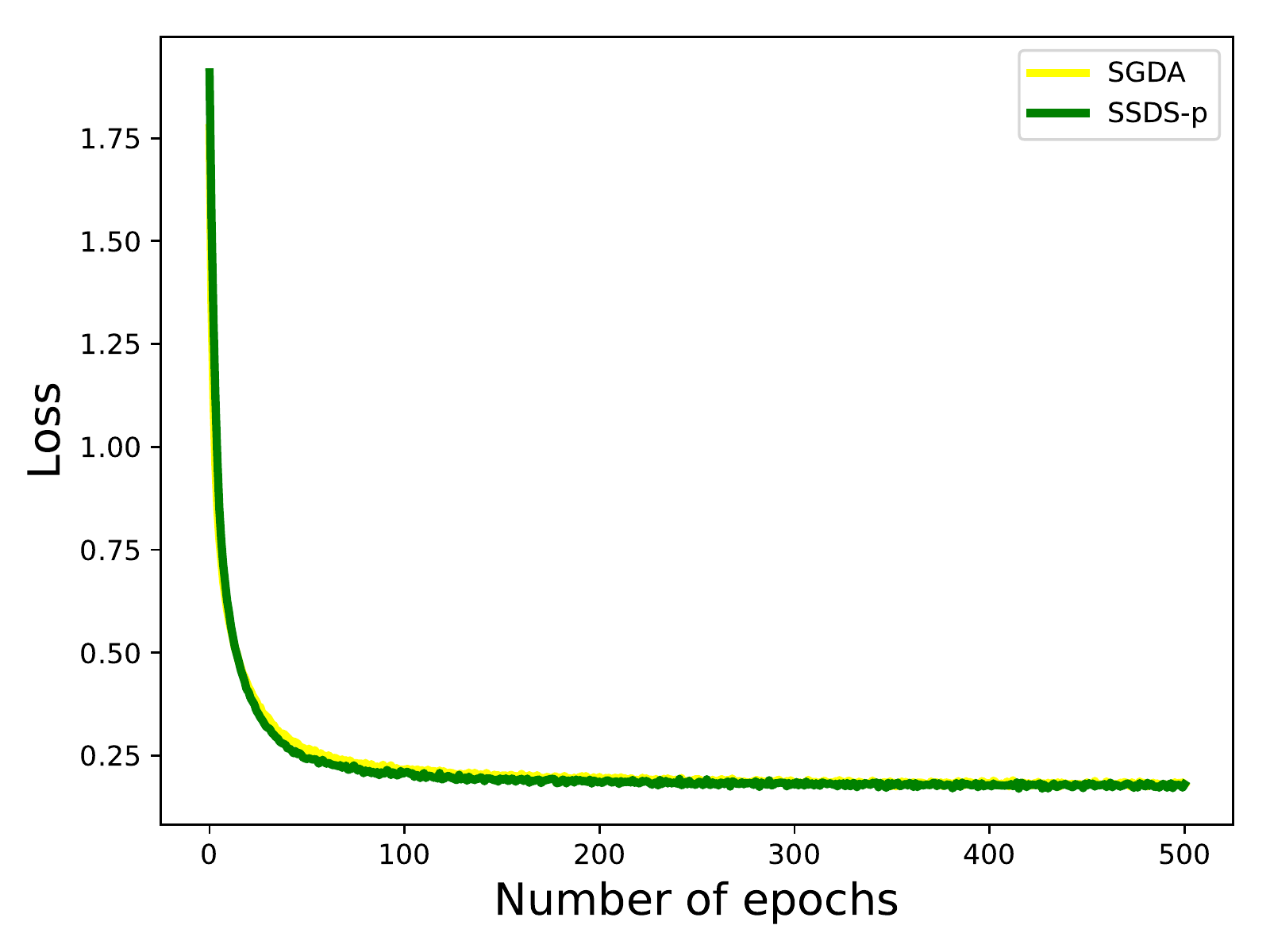}
        \caption{\label{fig:twoloss}}
    \end{subfigure}
    \caption{\textit{SGDA vs SSDS-p comparison on CIFAR-10 dataset, WideResNet model architecture}: (a) Accuracy (b) Loss value}\label{gdacomp}
\end{figure*}

\begin{figure*}[ht]
    \centering
    \begin{subfigure}{0.49\textwidth}
        \includegraphics[width=0.99\linewidth]{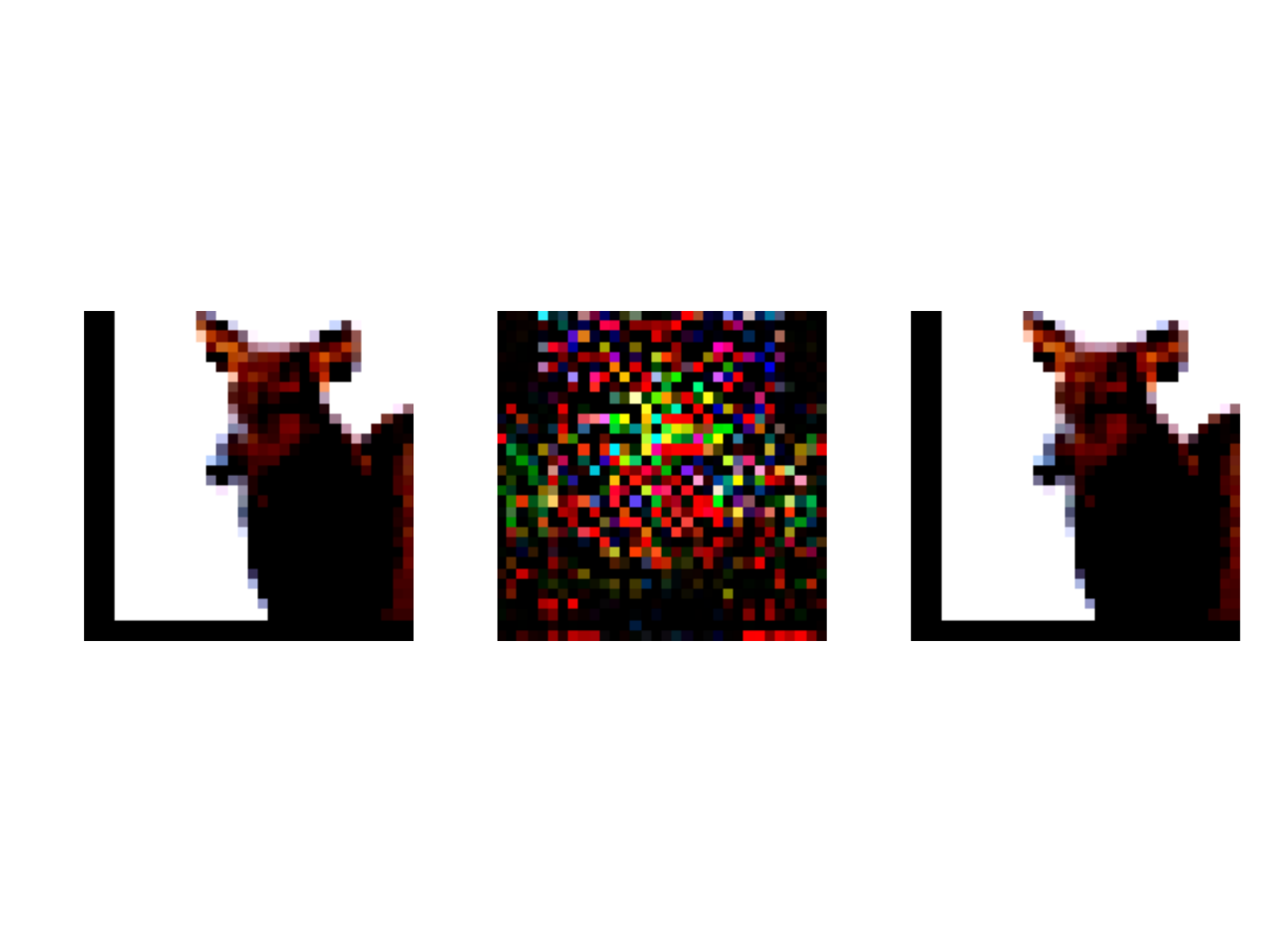}
    \end{subfigure}
    \begin{subfigure}{0.49\textwidth}
        \includegraphics[width=0.99\linewidth]{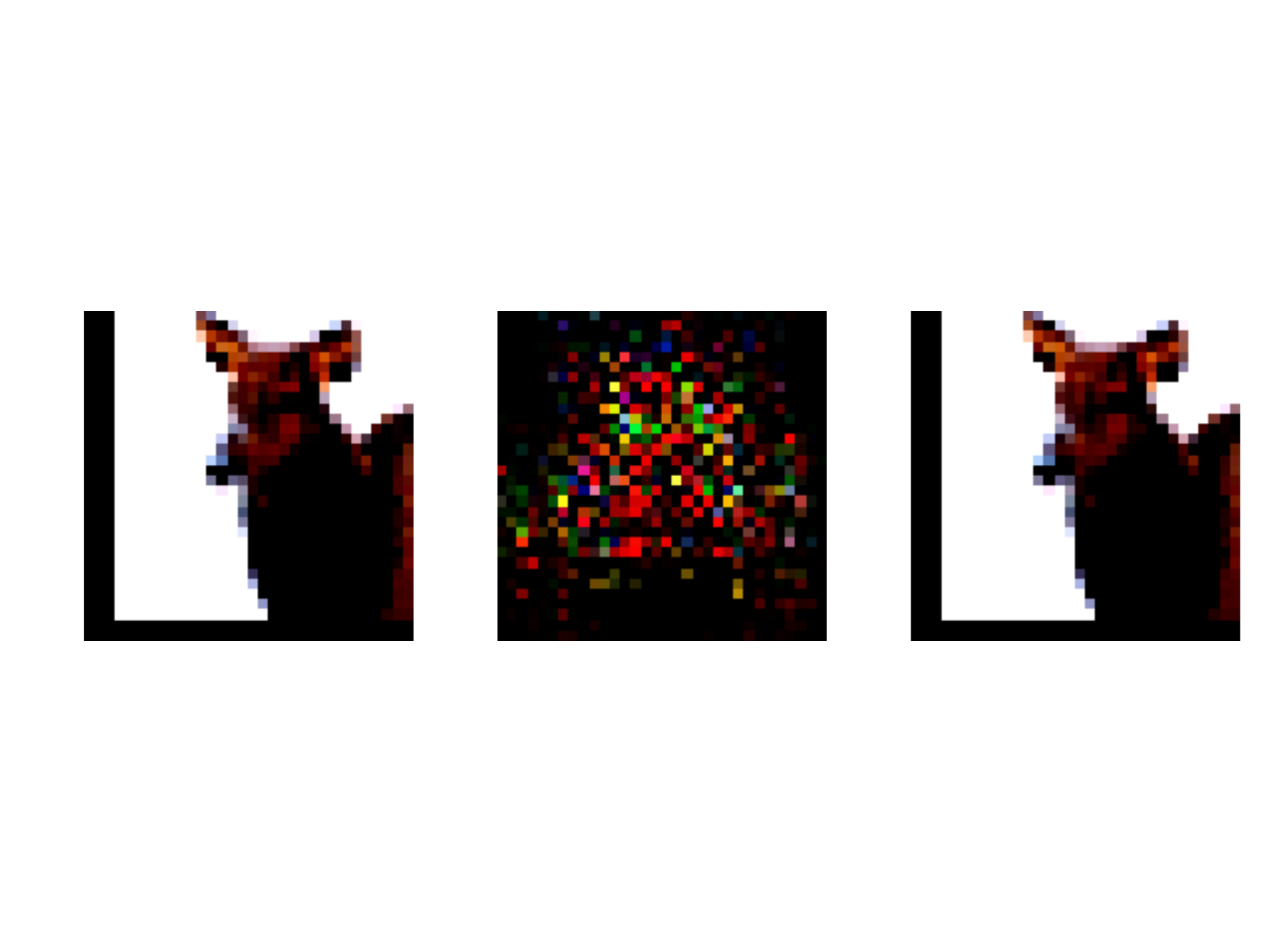}
    \end{subfigure}
    \vspace{-1pt}
    \begin{subfigure}{0.49\textwidth}
        \includegraphics[width=0.99\linewidth]{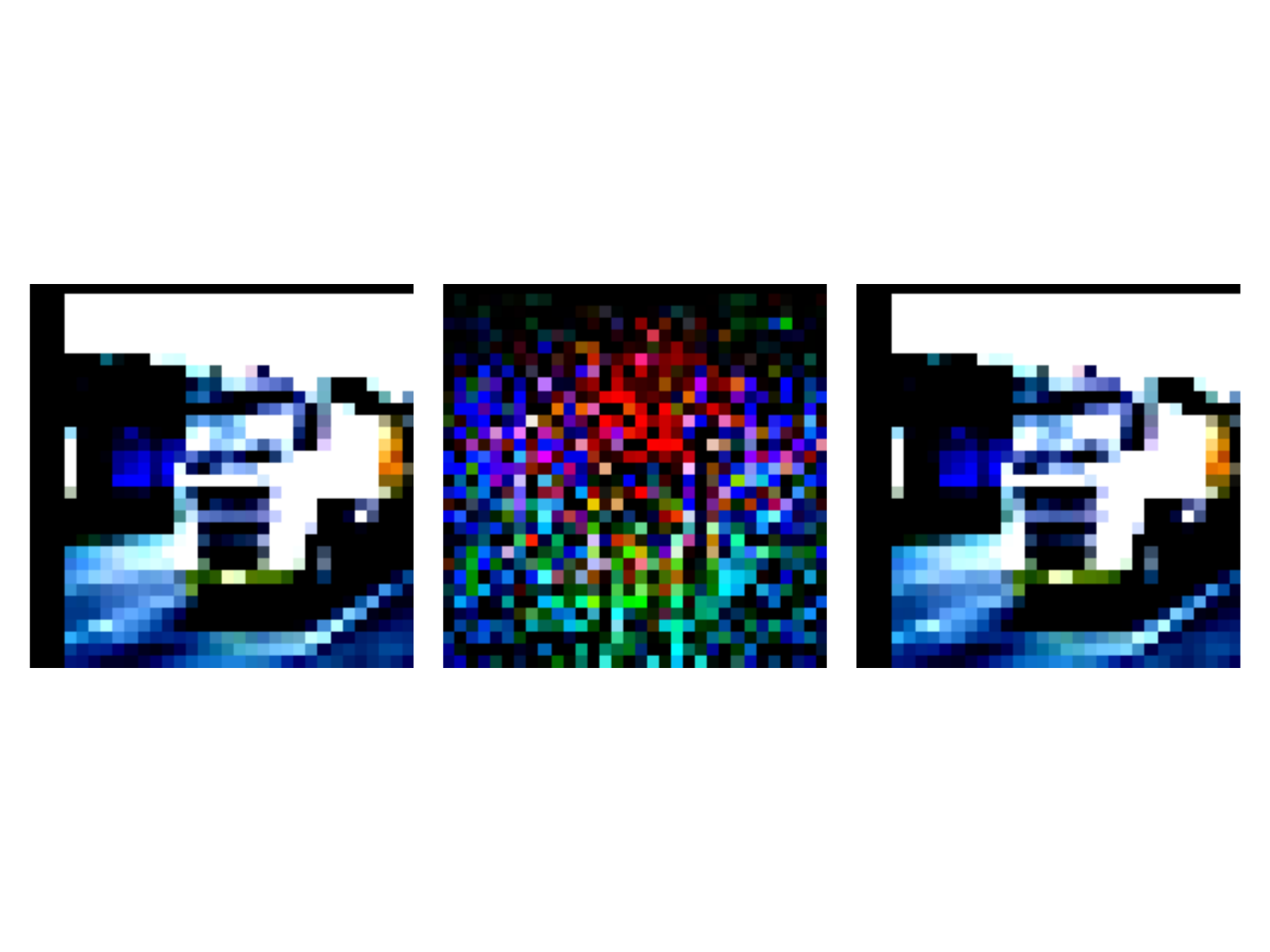}
    \end{subfigure}
    \begin{subfigure}{0.49\textwidth}
        \includegraphics[width=0.99\linewidth]{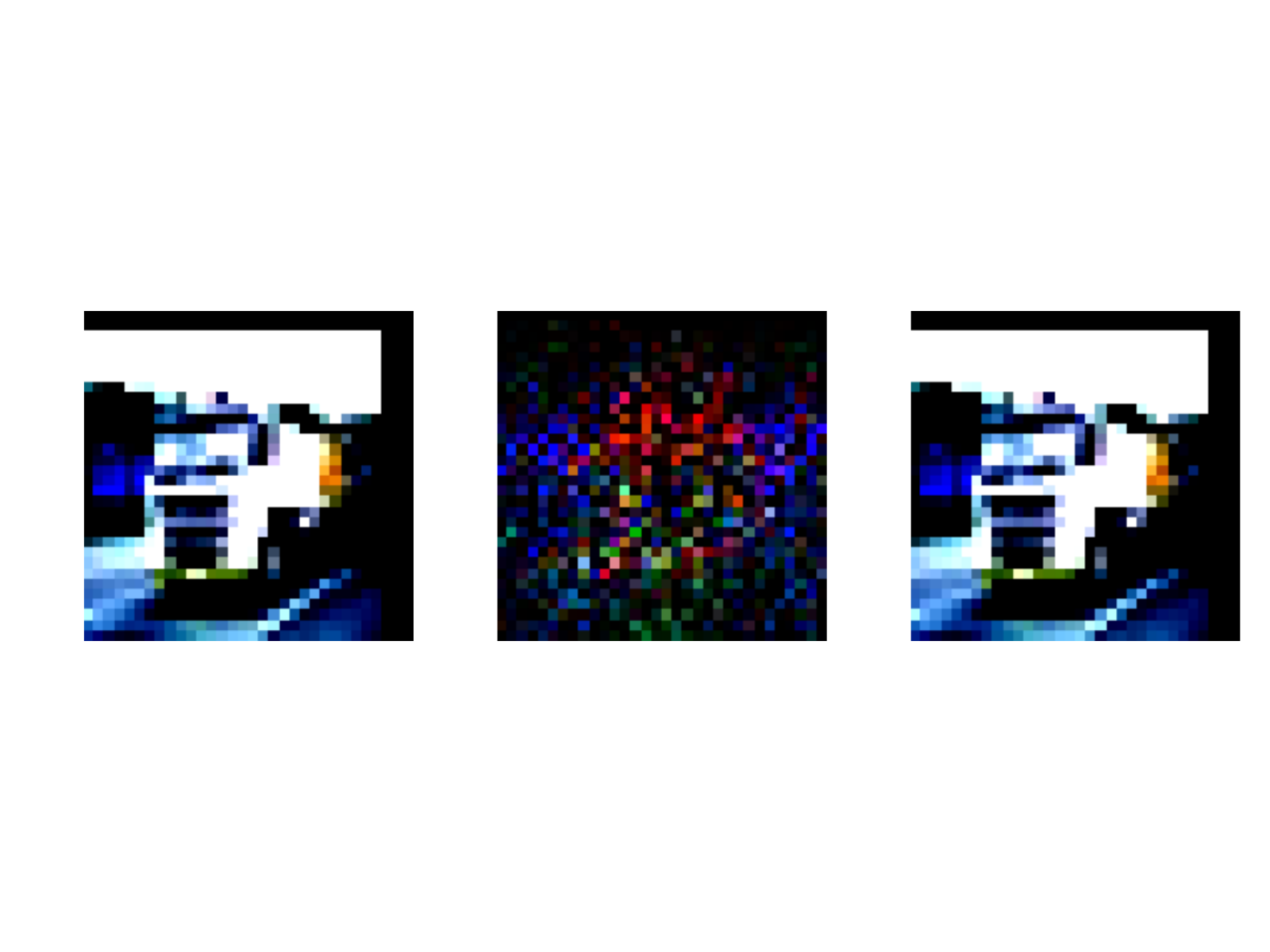}
    \end{subfigure}\\
    \begin{subfigure}{0.49\textwidth}
        \includegraphics[width=0.99\linewidth]{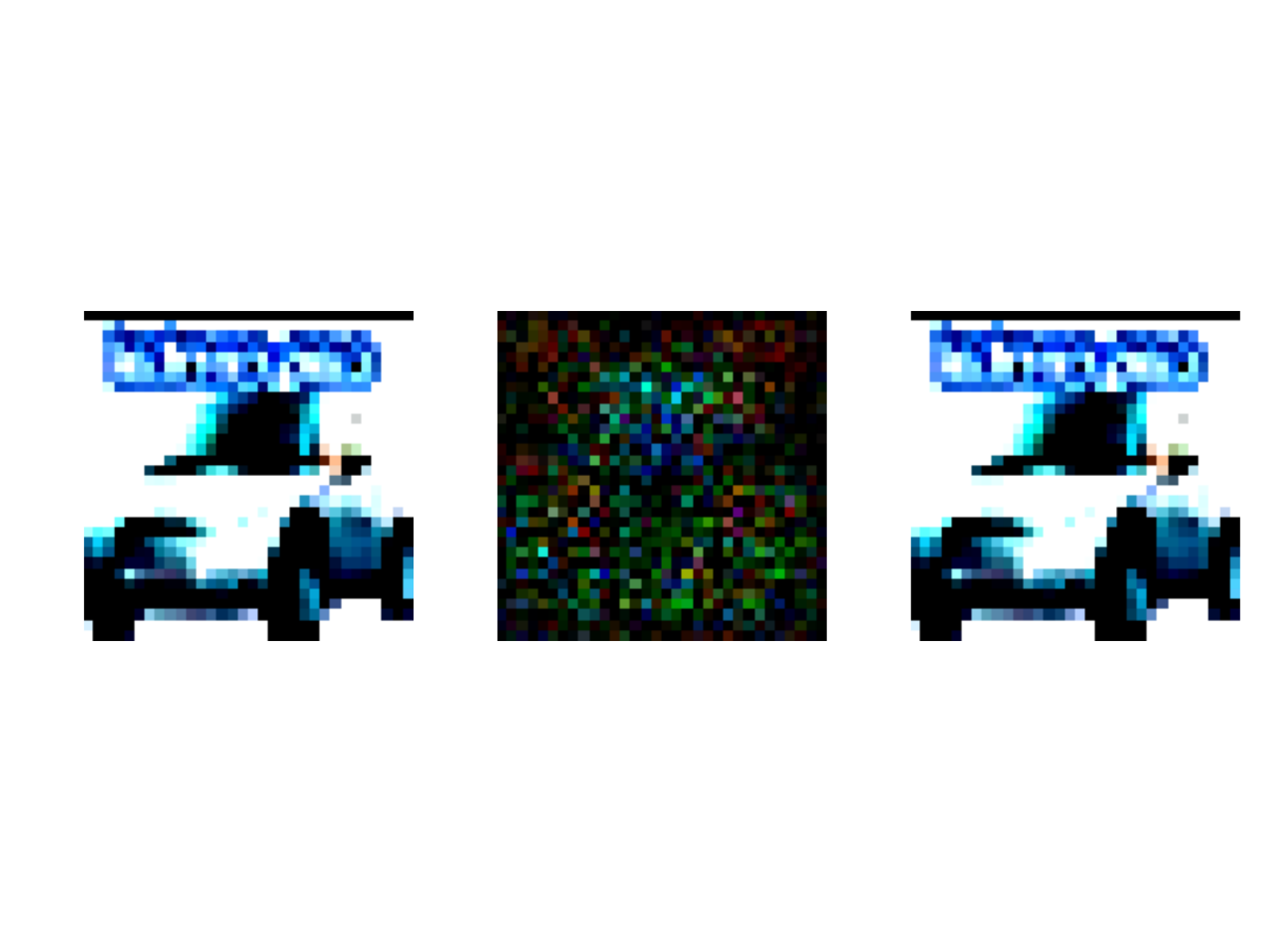}
        \caption{\label{fig:imgsgda}}
    \end{subfigure}
    \begin{subfigure}{0.49\textwidth}
        \includegraphics[width=0.99\linewidth]{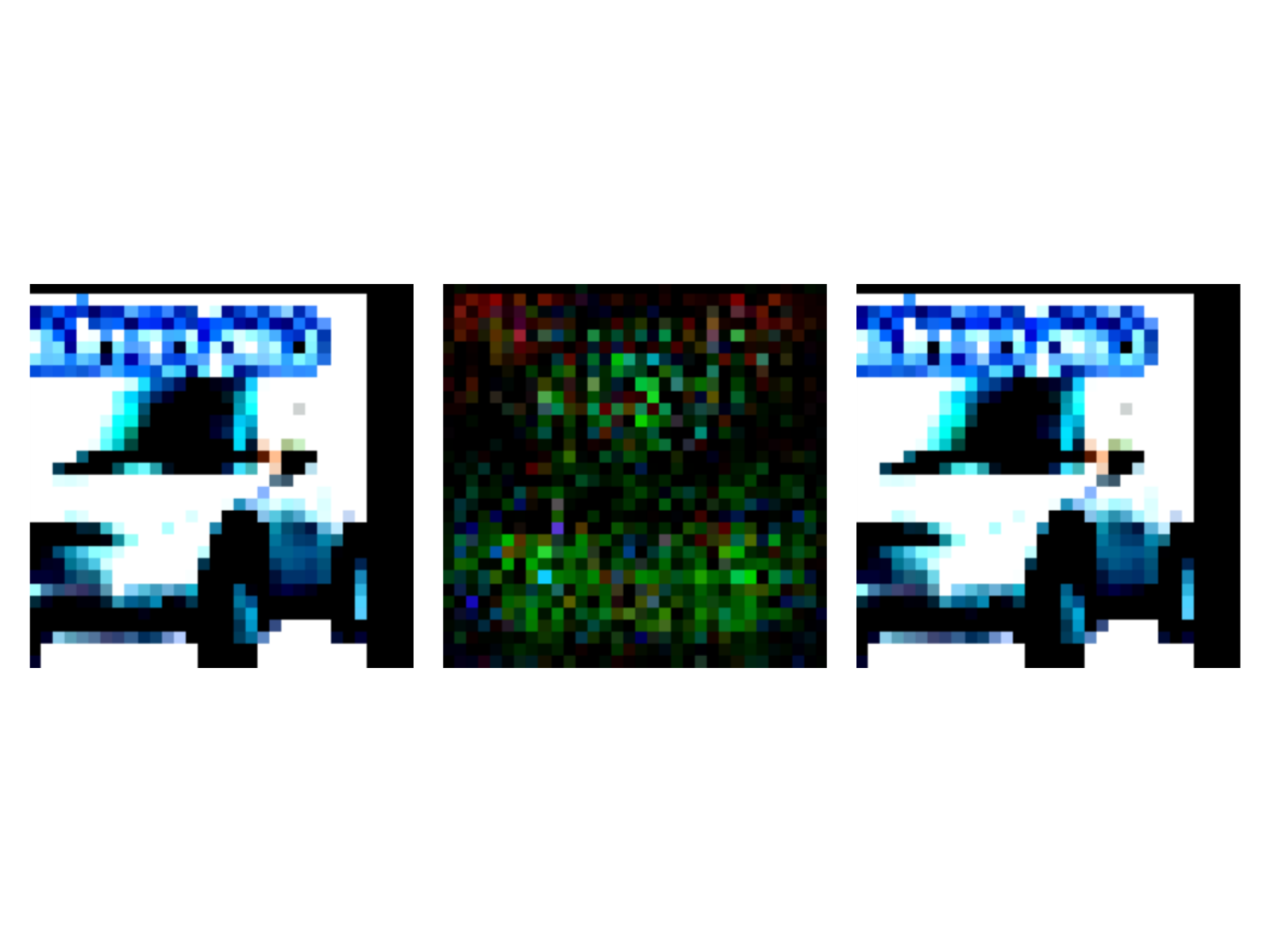}
        \caption{\label{fig:imgssdsp}}
    \end{subfigure}
    
    \caption{\textit{SGDA vs SSDS-p comparison on CIFAR-10 dataset, WideResNet model architecture}: (a) Randomly chosen image for SGDA training (left), with its corruption(center) and the corrupted image (right) (b) Randomly chosen image for SSDS-p training (left), with its corruption (center) and the corrupted image(right)}\label{viscomp}
\end{figure*}

\subsection{SSDS-p and SGDA attacks}
\vspace{10pt}
\begin{algorithm}
\caption{SGDA attack algorithm}
\begin{algorithmic}[1]\label{attacksgda}
     \STATE \textbf{Input}: $\varepsilon$, $p$, $w^\star$
     \STATE \textbf{Initialization}:$u_0$, $\alpha_0$
    \FOR{$k\in\{1,...,K\}$}
    \STATE $u_{k+1}^{(m_j)}=u_{k}^{(m_j)}+\alpha_k(\partial_{u_k}L(I^{(m_j)}+u_k^{(m_j)},y^{(m_j)},w^\star)$
    \STATE $\alpha_{k+1}=\alpha_k e^{-k\textit{p}}$
    \ENDFOR
\end{algorithmic}
\end{algorithm}

Throughout the paper, we discussed the performance of SSDS-p and SGDA algorithms as defense method against adversarial perturbations. However, SSDS-p and SGDA attacks can be generated through an iterative process by freezing the model ($w^\star$s) and following the same algorithm for updating $u$s. Note that, for generating SSDS-p attacks $\lambda$ is not getting updated too as it corresponds to the model parameters, whereas $v$s are being updated in the same fashion in SSDS-p training. Algorithms~\ref{attacksgda}, \ref{attackssds} are the corresponding attack algorithms for SGDA and SSDS-p methods.

\begin{algorithm}[ht]
\caption{SSDS attack algorithm}
\begin{algorithmic}[1]\label{attackssds}
     \STATE \textbf{Input}: $\varepsilon$, $p$, $C_1$, $w^\star$
     \STATE \textbf{Initialization}:$u_0$, $v_0$, $\alpha_0$
    \FOR{$k\in\{1,...,K\}$}
    \STATE $v_{k+1}^{(m_j)}=v_k^{(m_j)}+\alpha_k(\Vert u_k^{(m_j)}\Vert_\infty-\varepsilon)$
    \STATE $u_{k+1}^{(m_j)}=u_{k}^{(m_j)}+\alpha_k(\partial_{u_k} L(I^{(m_j)}+u_k^{(m_j)},y^{(m_j)},w^\star)-C_1 \; v_k^{(m_j)}\sgn(u_k^{(m_j)}))$
    \STATE $\alpha_{k+1}=\alpha_k e^{-k\textit{p}}$
    \ENDFOR
\end{algorithmic}
\end{algorithm}

Fig.~\ref{testaccplots} shows the performance of a well trained (SGDA) model on SSDS-p and SGDA attacks. The plots show that the model performs well and the testing accuracy is high in the beginning when the optimal attack is still not found, but as the final perturbation is crafted, the accuracy stays fixed and the test accuracies under SGDA and SSDS-p optimal attacks are achieved. 

\begin{figure*} [ht]
\centering
    \begin{subfigure}{0.49\textwidth}
        \includegraphics[width=0.99\linewidth]{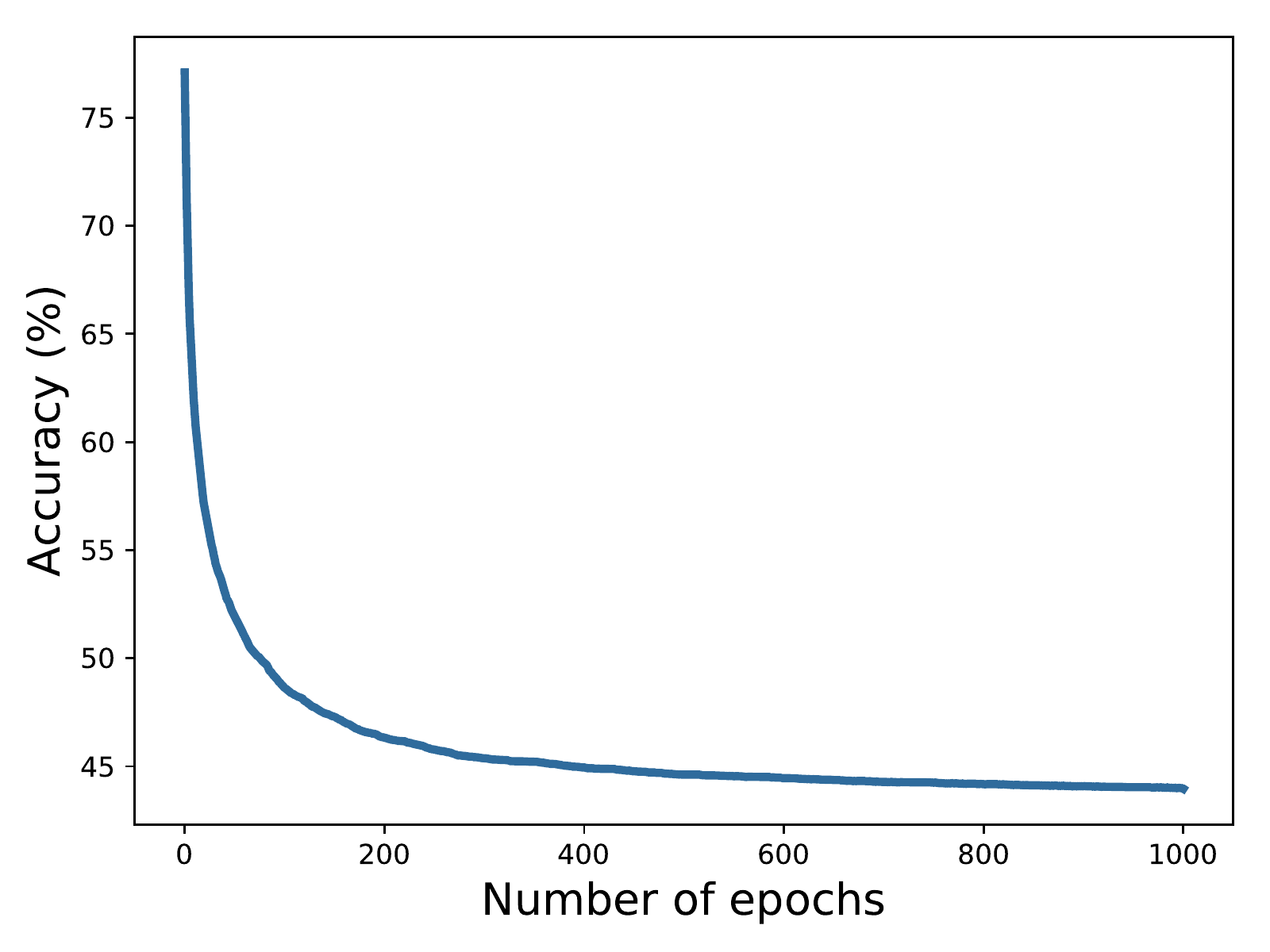}
        \caption{SGDA test accuracy}
        \label{fig:testsgda}
    \end{subfigure}
    \begin{subfigure}{0.49\textwidth}
        \includegraphics[width=0.99\linewidth]{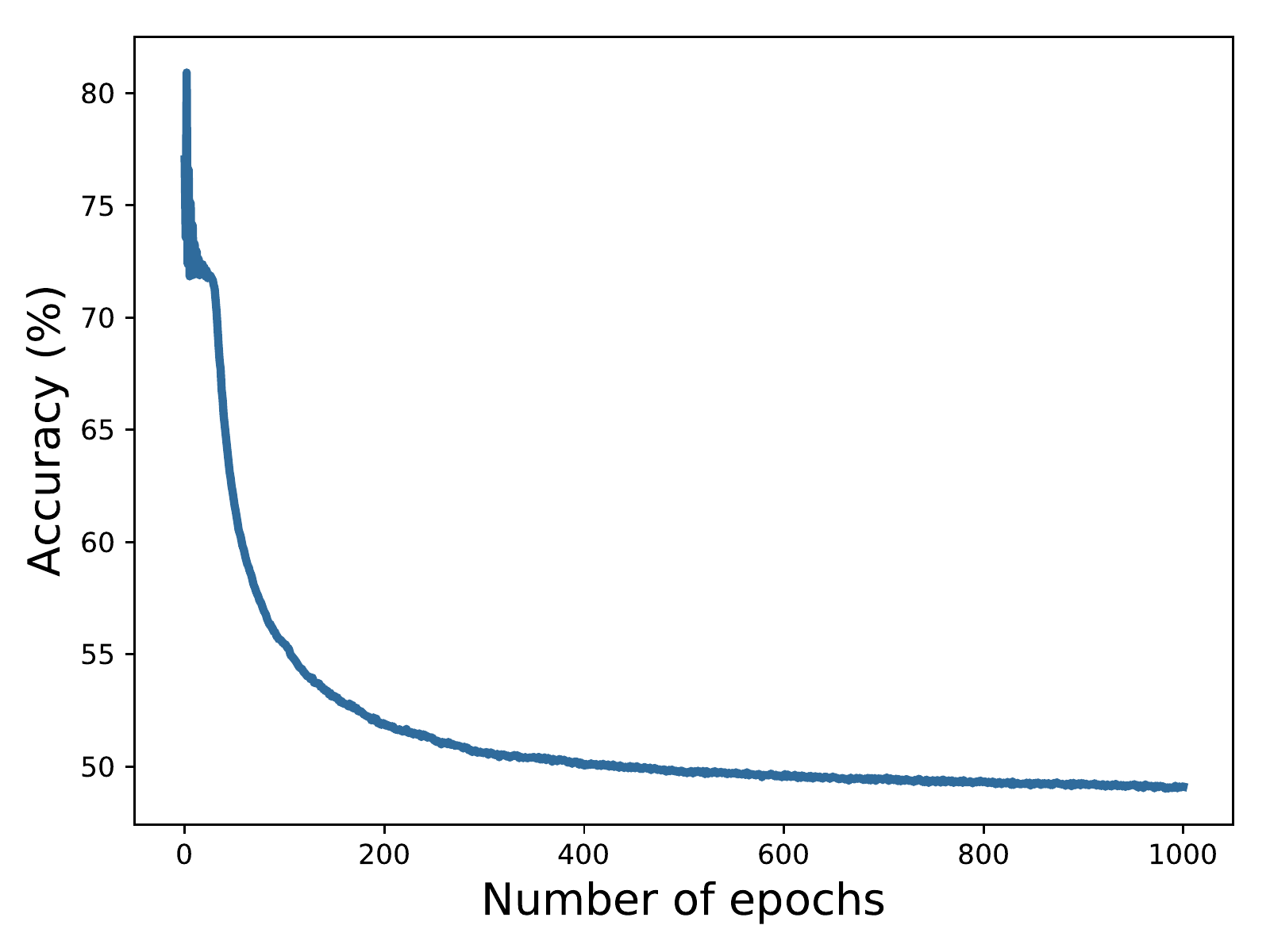}
        \caption{SSDS-p test accuracy}
        \label{fig:testssds}
    \end{subfigure}
    \caption{\textit{Testing accuracy plots  for SGDA robust model under (a) SGDA (b) SSDS-p attacks (dataset: CIFAR10, $\varepsilon = 0.03$, model architecture: VGG19)}}\label{testaccplots}
\end{figure*}

\subsection{SSDS-p performance evaluation and comparison}
In this section, we present a detailed performance evaluation of the proposed SSDS-p algorithm under both white-box and black-box attacks. We compare the performance of SSDS-p with the state-of-the-art robust learning algorithms. 

{\begin{table*} [ht]
\centering
    \caption{Performance comparison of several robust learning methods under white-box attacks using WideResNet on CIFAR-10 dataset}
    \resizebox{0.7\textwidth}{!}{{\begin{tabular}{c||c||c||c}
    Defence model & Clean & FGSM & PGD-20 \\
    \hline
    \hline
    Natural  & 94.82\% & 35.23\% & 5.93\% \\
    \hline
    FeatureScatter~\cite{zhang2019defense}  & 90.00\% & 78.40\% & 70.50\% \\
    \hline
    Free-8~\cite{shafahi2019adversarial} & 85.96\% & 53.57\%  & 46.82\% \\
    \hline
    PGD-7~\cite{madry2017towards} & 85.70\% & 54.90\% & 44.90\% \\
    \hline
    FGSM+DAWNBench~\cite{wong2020fast} & 83.12\% & 59.17\% & 45.18\%  \\
    \hline
    YOPO~\cite{zhang2019you} & 86.70\% & 55.23\% & 47.98\% \\
    \hline
    TRADES~\cite{DBLP:journals/corr/abs-1901-08573} & 84.92\% & 61.06\% & 56.61\%  \\
    \hline
    SGDA  & 81.97\% & 80.40\% & 44.32\% \\
    \hline
    SSDS-p & 82.91\%  & 81.21\% & 45.89\% \\
    \hline
    \end{tabular}}}
    \label{WhitetableWResNet10}
\end{table*}}

\textbf{CIFAR-10 dataset Under white-box attacks:} 

We present the performance of SSDS-p under white-box attacks (with FGSM and PGD attack models) in Table~\ref{WhitetableWResNet10} along with other state-of-the-art defense models. 
We also provide SGDA performance in the table as a baseline for our proposed approach. Based on the results presented, SSDS-p performs the best under FGSM attack, which is a computationally inexpensive attack.
Under computationally expensive attacks using the PGD-20 algorithm (i.e., PGD with 20 step iteration),  SSDS-p performance is comparable with the PGD-7, Free-8, FGSM+DAWNBench, and YOPO defence algorithms. However, as Table~\ref{Runningtime} suggests, while FeatureScatter, PGD, YOPO and TRADES training algorithms are significantly more computationally complex compared to SSDS-p, other algorithms (Free-8, FGSM+DAWNBench) have comparable accuracy and running time with SSDS-p. Comparing the performance of different algorithms in Table~\ref{WhitetableWResNet10}, we can see that only TRADES and FeatureScatter algorithm perform significantly better than all other methods under the PGD-20 attack. However, the methodology for generating the perturbations and using them eventually for training their robust model is different in FeatureScatter. In FeatureScatter~\cite{zhang2019defense}, adversarial examples are generated for training through feature scattering in the latent space, which is unsupervised in nature and computationally expensive. Similarly, while SSDS-p along with other methods listed in the table use cross entropy loss function, TRADES~\cite{DBLP:journals/corr/abs-1901-08573} differs in terms of the objective function and is still quite expensive computationally, similar to PGD training.


\textbf{CIFAR-10 dataset Under black-box attacks:}

We evaluate our proposed SSDS-p method under black-box attacks as well as compare head-to-head performance against PGD and TRADES approaches. As we wanted to use pre-trained PGD and TRADES models available online for this study, we had to use only ResNet50 architecture while comparing with the PGD-trained model and WideResNet architecture for comparison with the TRADES model. For the comparison with PGD training, we use the SSDS-p ResNet50 model for defending against $20$ step PGD attacks generated using a naturally trained model as well as using the pre-trained PGD model~\cite{robustness} available online as source models. We also evaluate the performance of the PGD above model under 20-step PGD attacks generated using the naturally trained and the SSDS-p models as source models. Results presented in Table~\ref{blacktable1} show that SSDS-p performs significantly better than the PGD model under black-box attacks. In order to compare with TRADES models, we use the pre-trained WideResNet TRADES model, which is available online. Based on similar experiments, it is evident that SSDS-p also performs better than the TRADES model under black-box attacks (as shown in Table~\ref{blacktable2}). 

\begin{table}[ht]
 \centering
 \caption{Black box performance: 20 step PGD attacks crafted by naturally trained model, SSDS-p trained model and PGD trained model~\cite{robustness} using ResNet50 architecture on CIFAR-10 dataset}
   \resizebox{0.6\textwidth}{!}{\begin{tabular}{c|c|c|c}
        \hline
         \textbf{Target} & \textbf{Source} & \textbf{Attack} & \textbf{Accuracy}\\\hline
                PGD-7  & naturally-trained & PGD-20 & 75.15\% \\ \hline
                SSDS-p & naturally-trained & PGD-20 &\textbf{78.59\%}\\ \hline
                PGD-7 & SSDS-p & PGD-20 & 71.62\% \\ \hline
                SSDS-p & PGD & PGD-20 &\textbf{78.53\%}\\\hline
    \end{tabular}}
    \label{blacktable1}
\end{table}

\begin{table}[ht]
 \centering
 \caption{Black box performance: 20 step PGD attacks crafted by naturally trained model, SSDS-p trained model and TRADES trained model~\cite{DBLP:journals/corr/abs-1901-08573} using WideResNet architecture on CIFAR-10 dataset}
  \resizebox{0.6\textwidth}{!}{\begin{tabular}{c|c|c|c}
        \hline
         \textbf{Target} & \textbf{Source} & \textbf{Attack} & \textbf{Accuracy}\\\hline
                TRADES & naturally-trained & PGD-20 & 66.97\% \\ \hline
                SSDS-p & naturally-trained & PGD-20 &\textbf{74.37\%}\\ \hline
                TRADES & SSDS-p & PGD-20 & 66.82\% \\ \hline
                SSDS-p & TRADES & PGD-20 &\textbf{80.63\%}\\\hline
    \end{tabular}}
    \label{blacktable2}

\end{table}


\textbf{CIFAR-100 dataset under white-box attacks:} 

We present the performance of SGDA, and SSDS-p under white-box attacks in Table~\ref{WhitetableWResNet100} and compare them with state-of-the-art methods. Similar to CIFAR-10 results, the performance of SSDS-p algorithm is better that Free-8 and PGD-7 robust models under FGSM attacks, and comparable with PGD-7 under PGD-20 attack. FeatureScatter still achieves the highest accuracy among all the methods but as we mentioned earlier it uses feature scattering techniques in the latent space to generate adversaries which makes it computationally expensive. 

\begin{table*} [ht]
\centering
    \caption{Performance comparison of several robust learning methods under white-box attacks using WideResNet on CIFAR-100 dataset}
    \resizebox{0.7\textwidth}{!}{{\begin{tabular}{c||c||c||c}
    Defence model & Clean & FGSM & PGD-20 \\
    \hline
    \hline
    Natural  & 74.00\% & 10.75\% & 0.00\% \\
    \hline
    FeatureScatter~\cite{zhang2019defense} & 73.90\% & 61.00\% & 47.20\% \\
    \hline
    Free-8~\cite{shafahi2019adversarial} & 62.13\% & 29.14\%  & 25.88\% \\
    \hline
    PGD-7~\cite{madry2017towards} & 59.90\% & 28.50\% & 22.60\% \\
    \hline
    SGDA  & 49.11\% & 48.01\% & 16.33\% \\
    \hline
    SSDS-p & 50.85\%  & 49.49\% & 19.01\% \\
    \hline
    \end{tabular}}}
    \label{WhitetableWResNet100}
\end{table*}

\textbf{CIFAR-100 dataset under black-box attacks:}

We evaluate our proposed SSDS-p method under black-box attacks and compare the performance with SGDA. We use the SSDS-p WideResNet model for defending against $20$ step PGD attacks generated using a naturally trained model as well as using the pre-trained SGDA model as source models. We also evaluate the performance of the above SGDA model under 20-step PGD attacks generated using the naturally trained and the SSDS-p models as source models. Results presented in Table~\ref{blacktable100} show that SSDS-p performs slightly better than the SGDA model under black-box attacks similar to the performance under white-box attacks.

\begin{table*}[ht]
 \centering
 \caption{Black box performance: 20 step PGD attacks crafted by naturally trained model, SSDS-p trained model and SGDA trained model using WideResNet model architecture on CIFAR-100 dataset}
   \resizebox{0.6\textwidth}{!}{\begin{tabular}{c|c|c|c}
        \hline
         \textbf{Target} & \textbf{Source} & \textbf{Attack} & \textbf{Accuracy}\\\hline
                SGDA  & naturally-trained & PGD-20 & 42.50\% \\ \hline
                SSDS-p & naturally-trained & PGD-20 &\textbf{44.08\%}\\ \hline
                SGDA & SSDS-p & PGD-20 & 41.12\% \\ \hline
                SSDS-p & SGDA & PGD-20 &\textbf{42.70\%}\\\hline
    \end{tabular}
    \label{blacktable100}}
\end{table*}

\textbf{SVHN dataset under white-box attacks:} 

We present the performance of SGDA, and SSDS-p under white-box attacks in Table~\ref{WhitetableWResNetSVHN} and compare them with stat-of-the-art methods. Results show that the performance of our SSDS-p algorithm is comparable to FeatureScatter under most of the attack schemes. We can provide more robust model compared to PGD-7 as well in significantly less time (refer to Table~\ref{Runningtime}).

\begin{table*} [ht]
\centering
    \caption{Performance comparison of several robust learning methods under white-box attacks using ResNet50 on SVHN dataset:}
    \resizebox{0.98\textwidth}{!}{{\begin{tabular}{c||c||c||c||c||c||c}
    Defence model & Clean & FGSM & PGD-20 & PGD-100 & CW-20 & CW-100 \\
    \hline
    \hline
    Natural & 97.20\% &  53.00\%  & 0.3\% & 0.1\% & 0.3\% & 0.1\%\\
    \hline
    PGD-7~\cite{madry2017towards} & 93.9\% & 68.4\% & 47.9\% &46.0\% & 48.7\% & 47.3\%\\
    \hline
    FeatureScatter~\cite{zhang2019defense} & 96.2\% & 83.5\% & 62.9\% & 52.0\% & 61.3\% & 50.8\% \\
    \hline
    SGDA & 93.97\% & 69.82\% & 52.52\%  &  52.34\% & 52.42\% & 52.35\% \\
    \hline
    SSDS-p & 94.32 \% & 73.53\% & 61.39\% & 61.33\% & 61.07\%& 60.94 \\
    \hline
    \hline
    \end{tabular}}}
    \label{WhitetableWResNetSVHN}
\end{table*}

\textbf{SSDS-p with different model architectures:}

To further analyse the algorithm, we have evaluated the performance of our model on other model architectures (e.g., ResNet50 and VGG19). Table~\ref{WhitetableRes5010} summarizes the performance of SSDS-p using ResNet50 model architecture along with other training methods on CIFAR-10 and CIFAR-100 datasets. Results align with Table~\ref{WhitetableWResNet10} where SSDS-p performs the best under FGSM attack and the the model is comparable to PGD-7 model under PGD-20 attacks. Similarly, results on CIFAR-100 dataset shows that SSDS-p performs the best compared to SGDA and naturally trained models using ResNet model architecture.

Additionally, as the model architectures play a significant role specially when the input data is large, we run SSDS-p on different model architectures to analyse the dependency of SSDS-p on the model capacity, gradients, etc. 
Based on Table~\ref{modelcomp}, we see that the model capacity and complexity definitely helps both natural and robust accuracy, although the effect is not as significant as its effect on PGD robust models reported by~\citet{madry2017towards}. 

\begin{table*} [ht]
\centering
    \caption{Performance comparison of several training methods under white-box attacks using ResNet50 on CIFAR-10 dataset}
     \resizebox{0.7\textwidth}{!}{\begin{tabular}{c||c||c||c||c}
    Dataset & Defence model & Clean & FGSM & PGD-20\\
    \hline
    \hline
    CIFAR-10 & Natural & 92.70\% &  27.50\%  & 0.82\%\\
    \hline
    CIFAR-10 & PGD-7 & 79.4\% &  51.7\%  & 43.7\%\\
    \hline
    CIFAR-10 & SGDA & 78.04\% & 76.42\% & 40.70\% \\
    \hline
    CIFAR-10 & SSDS-p & \textbf{80.1\%} & \textbf{79.32\%} & 42.11 \% \\
    \hline
    \hline
    CIFAR-100 & Natural & 47.96\% &  43.66\%  & 13.01\%\\
    \hline
    CIFAR-100 & SGDA & 47.07\% & 46.18\% & 18.16\% \\
    \hline
    CIFAR-100 & SSDS-p & \textbf{49.77\%} & \textbf{49.22\%} & \textbf{19.69 \%} \\
    \hline
    \hline
    \end{tabular}}
    \label{WhitetableRes5010}
\end{table*}


\begin{table*}[ht]
    \caption{SSDS-p defense with different model architectures under white-box attacks}
    \begin{subtable}{0.32\linewidth}
        \centering
        \resizebox{0.7\textwidth}{!}{
        \begin{tabular}{c|c}
            Attack & Accuracy\\
            \hline
            Clean & 77.13\% \\
            FGSM & 75.92\% \\
            PGD & 42.32\% \\
            SGDA & 43.49\% \\
            SSDS-p & 47.56\% \\
        \end{tabular}
        }
      \caption{VGG 19}
      \label{tab:VGG}
    \end{subtable}
    \hfill
    \begin{subtable}{0.32\linewidth}
        \centering
        \resizebox{0.7\textwidth}{!}{
      \begin{tabular}{c|c}
            Attack & Accuracy\\
            \hline
            Clean & 80.10 \% \\
            FGSM & 79.32 \% \\
            PGD & 42.11\% \\
            SGDA & 44.58 \% \\
            SSDS-p & 50.09 \% \\
        \end{tabular}
        }
        \caption{ResNet50}
        \label{tab:50}
     \end{subtable}
    \begin{subtable}{0.3\linewidth}
        \centering
        \resizebox{0.7\textwidth}{!}{
      \begin{tabular}{c|c}
            Attack & Accuracy\\
            \hline
            Clean & 82.91 \% \\
            FGSM & 81.21 \% \\
            PGD & 45.89\% \\
            SGDA & 49.18 \% \\
            SSDS-p & 53.53 \% \\
        \end{tabular}
        }
        \caption{WideResNet}
        \label{tab:wide}
     \end{subtable}
    \label{modelcomp}
\end{table*}

\textbf{Attacks with more number of steps} 

To further evaluate our algorithm, We present the performance of SGDA, and SSDS-p under PGD and CW white-box attacks with various steps. The results are summarized in Tables~\ref{WhitetableMoreAttacks10} for CIFAR-10 and CIFAR-100 datasets. Similar to other robust models, SSDS-p can preserve the performance when the number of steps increase (the generated attacks become more powerful).


\begin{table*} [ht]
\centering
    \caption{Performance comparison of several robust learning methods under white-box attacks using WideResNet}
    \resizebox{0.99\textwidth}{!}{{\begin{tabular}{c||c||c||c||c||c||c}
    Dataset & Defence model & PGD-10 & PGD-40 & PGD-100 & CW-20 & CW-100\\
    \hline
    \hline
 CIFAR-10 & Natural & 7.83\% & 4.18\% & 0\% & 4.73\% & 0.32\%  \\
    \hline
 CIFAR-10 & FeatureScatter~\cite{zhang2019defense}  & 70.90\% & 70.3\% & 68.60\% & 62.40\% & 60.6\%  \\
    \hline
 CIFAR-10 & PGD-7~\cite{madry2017towards} & 45.10\% & 44.80\% & 44.80\% & 45.70\% & 45.40\% \\
    \hline
 CIFAR-10 & SGDA  & 45.74\% & 44.05\% & 44.03\% & 44.82\% & 44.66\% \\
    \hline
 CIFAR-10 & SSDS-p & 47.35\% & 45.61\% & 46.44\% & 46.62\% & 46.40\% \\
    \hline
    \hline
 CIFAR-100 & Natural & 12.83\% & 9.61\% & 6.42\% & 12.13\% & 7.21\%  \\
    \hline
 CIFAR-100 & SGDA  & 18.85\% & 16.09\% & 15.76\% & 16.52\% & 16.37\% \\
    \hline
 CIFAR-100 & SSDS-p & 19.23\% & 16.15\% & 16.10\% & 16.78\% & 16.70\% \\
    \hline
    \end{tabular}}}
    \label{WhitetableMoreAttacks10}
\end{table*}


\subsection{Computational time comparison}
One of the key advantages of our proposed algorithm is the significantly lower computational overhead compared to the standard techniques for training robust models. To demonstrate this, we compare computation times with experiments run on one node of a GPU cluster with Intel Xeon Processor with 32 cores, 128GB RAM and with a Titan X (Pascal architecture) GPU having 12GB of GPU RAM. Fig.~\ref{fig:runtime} shows that although PGD-7 and SSDS-p achieve comparable robust classification accuracy for CIFAR-10 dataset, the training time per epoch for 7 step PGD-training algorithm is approximately $4$ times greater than that of the SSDS-p training algorithm. On the other hand, SSDS-p training takes around the same time as a $1$ step PGD training (or FGSM-based adversarial training).
Table~\ref{Runningtime} summarizes the running time corresponding to other state-of-the-art algorithms. SSDS-p (similar to Free-8 and FGSM+DAWNBench) can achieve comparable accuracy with PGD in significantly less time. FeatureScatter which is the most accurate compared to other algorithms is computationally very expensive. In summary, our approach is a computationally efficient framework for robust learning compared to most of the  state-of-the-art techniques.

\begin{figure}[ht]
\centering
        \includegraphics[width=0.5\linewidth,clip,trim={0in 0.1in 0.0in 0.5in}]{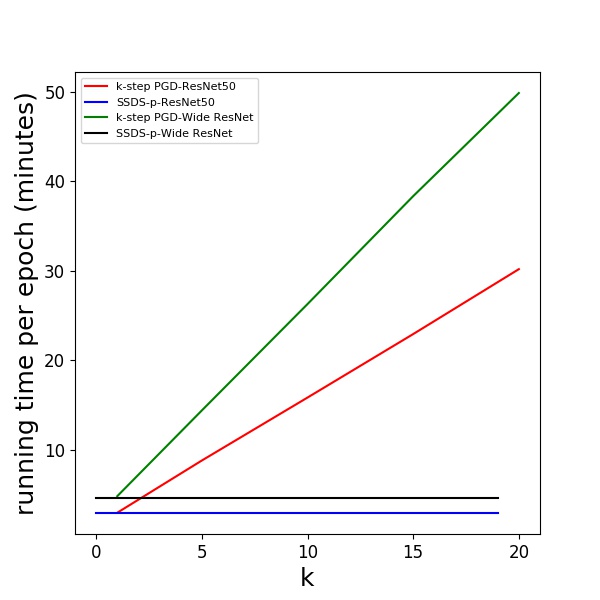}
        \caption{Training time per epoch for PGD vs. SSDS-p}
        \label{fig:runtime}
\end{figure}

\begin{table*} [ht]
\centering
    \caption{Running time comparison of several robust learning methods using WideResNet and Cifar-10 dataset:}
    \resizebox{0.7\textwidth}{!}{{\begin{tabular}{c||c}
    Defence model & Time (seconds per epoch) \\
    \hline
    \hline
    Free-8~\cite{shafahi2019adversarial} & \textbf{289}\\
    \hline
    FeatureScatter~\cite{zhang2019defense} & 3132 \\
    \hline
    FGSM+DAWNBench~\cite{wong2020fast} & \textbf{296}\\
    \hline
    YOPO~\cite{zhang2019you} & 790\\
    \hline
    FGSM~\cite{goodfellow2014explaining} & 295 \\
    \hline
    PGD-7~\cite{madry2017towards} & 1269 \\
    \hline
    PGD-20~\cite{madry2017towards} & 3011\\
    \hline
    TRADES~\cite{DBLP:journals/corr/abs-1901-08573} & 1006\\
    \hline
    SSDS-p & \textbf{291}\\
    \hline
    \hline
    \end{tabular}}}
    \label{Runningtime}
\end{table*}

\section{Conclusion}\label{sec6}
In this paper, we propose a new stochastic saddle-point dynamical systems approach to solve the robust learning problem. Under certain restrictive assumptions, we present a detailed convergence analysis of our algorithm. Our proposed algorithm involves a Lagrangian formulation to solve the robust optimization problem, where we introduce two Lagrangian multipliers for both model parameters and uncertainties. The multiplier for the uncertainties allows us to handle more complex uncertainty constraints where the uncertainties are assumed to belong to more general convex sets, for example ellipsoids or intersections of ellipsoids. In the absence of a multiplier to impose these constraints on the uncertainty set, one will have to resort to solving an optimization problem within another optimization problem, thereby making the rigorous convergence proof of such algorithm quite difficult. Our approach is useful even for the case when the uncertainty has a simple constraint. For example, an $\ell_\infty$ norm bounded uncertainty can be imposed by simple projection. However, obtaining rigorous convergence proof of an algorithm consisting of a multiplier to impose the parameter constraints and simple projection to impose the uncertainty constraints is difficult and does not exist. Similar justification applies to using Lagrangian multipliers for model parameters, especially since the constraints on the model parameters are more complicated than simple box constraints. In addition to these advantages, we also observe performance improvement over the baseline SGDA algorithm that does not use the Lagrangian multipliers. \color{black} Empirically, we show that the proposed scheme is a computationally inexpensive method that maintains a high level of performance for clean and corrupted input data, both for white-box and black-box attacks. We believe that this can be attributed to the adversarial training in SSDS also acts as a form of regularization. We can explain it based on the equivalence between the robust optimization problem and many regularization problems~\cite{sra2011}. Finally, we note that this is an early attempt to adopting a dynamical systems approach to robust learning. Future research will focus on relaxing some of the restrictive assumptions in the analysis for the loss function and uncertainties. Similarly, we will focus on developing the SSDS algorithm further by modifying the cost functions to better handle the highly non-convex nature of deep network loss functions, leading to performance improvement. 


\clearpage
\newpage
\bibliographystyle{unsrtnat}
\bibliography{ref}

\clearpage
\newpage

\section{Supplementary Material}
Additional theoretical analysis, convergence proofs and experimental results are discussed here; also more information about the material in the main body of the paper is presented.

\subsection{Additional Theorems and Derivations}

As we discussed in the main paper, the robust learning problem is:
\begin{align} \label{ROs}
{\cal RO}:=&\min_{w} \; \mathbb{E}_{(I,y) \sim {\cal D}} \; \Big[\underset{u \in {\cal U}}{\max}\;\; L(I+u,y,w)\Big]\;
\end{align}
and the total Lagrangian can be written as: 
\begin{align}
{\cal L}(x,\lambda,u,v):=&\; t+\lambda \;\Big( \textstyle \sum_{i=1}^N \big(L(I^{i}+u^{i},y^{i},w) \nonumber\\ 
& \; -v^{i} \; h^{i}(u^{i}) \big)-t \Big) \;
\label{lags}
\end{align}
Then we propose the following update rules for the parameters:
\begin{align}
&x_{k+1}=x_k-\alpha_k (\partial_x f(x_k)+\lambda_k \partial_x g(x_k,u_k,\xi_k))\label{sspd1s}\;,\\
&\lambda_{k+1}=\big[\lambda_k+\alpha_k \big(g (x_k,u_k,\xi_k)-\textstyle \sum_{i=1}^N v_k^{i} h^{i}(u^{i}_k)\big)\big]_+\label{sspd2s}\;,\\
&u^{i}_{k+1}=u^{i}_k+\alpha_k (\partial_{u^{i}} g(x_k,u_k)-v_k^{i} \partial_{u^{i}} h^{i}(u^{i}_k))\label{sspd3s}\;,\\
&v^{i}_{k+1}=[v_k^{i}+\alpha_k \lambda_k h^{i}(u^{i}_k)]_+\;\;i=1,\ldots,N.\label{sspd4s}
\end{align}
where, $[\cdot]_{+}$ is positive projection, $\xi_k$ is assumed to be an independent and identically distributed random process and $\alpha_k$ is the adaptive step-size with the following characteristics
\vspace{-10pt}
\begin{align}\label{alphas}
    &\alpha_k=\frac{\gamma_k}{\Vert T(z_k)\Vert_{2}}\;,{\rm with} \;\;\gamma_k>0\;,\nonumber\\
&\sum_{k=1}^{\infty}\gamma_k=\infty\;,\; \textstyle \sum_{k=1}^{\infty} \gamma_k^{2}<\infty\;.\vspace{-5pt}
\end{align}

Under the following assumption on $f$, $g$ and $h^{i}$.  
\begin{assumption}\label{assume_ssds2}
We assume that 
$f(x)$ is  convex in $x$ and each $h^{i}(u^{i})$ is convex in $u^{i}$.
Moreover, $g(x,u,\xi)$ is convex in $x$ and is strictly concave in $u$ for any fixed value of \;$\xi$.  
\end{assumption}

Then the following theorem is the main result for asymptotic convergence of the discrete-time saddle point algorithm with diminishing step-size. We show that the update rules in (\ref{sspd1s})-(\ref{sspd4s}) lead to convergence to the KKT point (equivalent to the saddle point as specified in the supplementary material) of the ${\cal RO}$ problem.

\begin{theorem}
Let Assumption~\ref{assume_ssds2} hold and we also assume that $\lambda^\star>0$ where $\lambda^\star$ is the saddle point of the Lagrangian~(for $\lambda$ in Eq.\ref{lags}), then, following is true for the SSDS algorithm with adaptive step-size $\alpha_k$ satisfying Eq.~\ref{alphas}.
\begin{align}
\lim_{k\to \infty}\mathbb{E}_{\xi_0^k}[x_k]=x^\star,&
\lim_{k\to \infty}\mathbb{E}_{\xi_0^k}[u_k]=u^\star\;,\\
\text{where}\quad \xi_0^k&=\{\xi_0,\ldots,\xi_k\}.\nonumber
\vspace{-20pt}
\end{align}
\end{theorem}

To elaborate more on the above derivations, We consider the following definitions in this section:

\begin{align} \label{fgh}
x:=(t,w),\ f(x):=t,\ g(x,u):=L(I+u,y,w)-t\;.
\end{align}

\subsection*{Lagrangian function derivation}

First, we derive the Lagrangian function in (\ref{lags}). The lower level optimization problem in the ${\cal RO}$ problem (\ref{ROs}) can be written as an optimization problem parametrized by $x$ such that

\begin{align}
\theta(x):=\textstyle \sum_{i=1}^N \underset{u^{(i)} \in {\cal U}^{(i)}}{\max}\; L(I^{(i)}+u^{(i)},y^{(i)},w)-t\;.\label{lower_ro}
\end{align}

Denote $v^{(i)}$ s as Lagrangian multipliers for the lower level maximization problem and define $v:=[v_{1},\ldots, v_{N}]^\top$. The role of each $v^{(i)}$ multiplier is to satisfy the uncertainty set constraint associated with the perturbation $u^{(i)}$. Based on the Lagrangian theory, one can equivalently write

\begin{align} \label{lower_level}
\theta(x)=\textstyle \sum_{i=1}^N \underset{u^{(i)}}{\max}\;\underset{v^{(i)}\geq0}{\min} \; \big(L(I^{(i)}+u^{(i)},y^{(i)},w)\nonumber\\
-v^{(i)} \; h^{(i)}(u^{(i)})\big)-t \; .
\end{align}

Hence, ${\cal RO}$ problem can be written as

\begin{align}
{\cal RO}=&\;\underset{x=(w,t)}{\min}\;\underset{\lambda\geq0}\max\; \Big\{ \; t \; + \nonumber \\ & \lambda \;\Big(\textstyle \sum_{i=1}^N \underset{u^{(i)}}{\max}\;\underset{v^{(i)}\geq0}{\min}\;\big(L(I^{(i)}+u^{(i)},y^{(i)},w)
\;- \nonumber \\ &v^{(i)} \; h^{(i)}(u^{(i)})\big)-t \Big) \; \Big\}\\
=&\;\underset{x=(w,t)}{\min}\;\underset{\lambda\geq0}\max\;\underset{u^{(i)}}{\max}\;\underset{v^{(i)}\geq0}{\min}\; \Big\{ \; t \; + \nonumber \\ & \lambda \;\Big( \textstyle \sum_{i=1}^N \big(L(I^{(i)}+u^{(i)},y^{(i)},w)
-v^{(i)} \; h^{(i)}(u^{(i)}) \big)-t \Big)\; \Big\} \; .\nonumber
\end{align}

Therefore, one can derive the total Lagrangian as in (\ref{lags}).

\subsection*{Saddle and KKT point of the ${\cal RO}$ problem}
The following theorem can be stated for the saddle point of the optimization problem (\ref{ROs}).
\begin{theorem}\label{theorem_saddle}
Consider the Lagrangian function as defined in (\ref{lags}). Under Assumption \ref{assume_ssds2}, following statements are true for the optimization problem (\ref{ROs})

\begin{align*}
&\mu=\underset{x}{\min}\;\underset{\lambda\geq0}\max\;\underset{u^{(i)},\forall i}{\max}\;\underset{v^{(i)}\geq0,\forall i}{\min}\; {\cal L}=\underset{x}{\min}\;\underset{v_{i\geq0},\forall i}{\min}\;\underset{\lambda\geq0}\max\;\underset{u^{(i)},\forall i}{\max}\;{\cal L}\;,\\
&\mu=\underset{x, v^{(i)},\forall i}{\min}\;\underset{\lambda\geq0, u^{(i)},\forall i}{\max}\; {\cal L}=\underset{\lambda\geq0, u^{(i)},\forall i}{\max}\;\underset{x, v^{(i)},\forall i}{\min}\; {\cal L}\;,
\end{align*}

where $x\in \mathbb{R}^n,\lambda\geq 0, u^{(i)}\in \mathbb{R}^{m}$, and $v^{(i)}\geq 0$ for $i=1,\ldots, N$\;. Hence, the Lagrangian function (\ref{lags}) has a saddle point.
\end{theorem}
\begin{proof}
For the ease of notations, consider the ${\cal ROs}$ problem with single constraint and single uncertainty set as

\begin{align}
\mu:=&\;\underset{x}{\min}\;\; f(x) \;\;\text{s.t.} \;\; \underset{h(u) \leq 0}{\max}\;\; g(x,u)\leq 0\;.\label{RO_single}
\end{align}

The general case (\ref{ROs}) with multiple uncertainty sets can be proved along similar lines. The Lagrangian for upper level problem in (\ref{RO_single}) is 

\begin{align*}
f(x)+\lambda\left(\max_{h(u)\leq 0} g(x,u)\right)\;.
\end{align*}

We can write the total Lagrangian for (\ref{RO_single}) as
\begin{align*}
{\cal L}(x,v,\lambda,u)=f(x)+\lambda\left(g(x,u)-vh(u)\right)\;.
\end{align*}
Hence, we can write
\begin{align*}
\mu&= \underset{x}{\min}\; \underset{\lambda\geq 0}{\max}\; \underset{u}{\max}\;\underset{v\geq 0}{\min}\;{\cal L}(x,v,\lambda,u)\\
&=\underset{x}{\min}\max_{\lambda\geq 0}\;\underset{u}{\max}\;\underset{v\geq 0}{\min} \;\big(f(x)+\lambda(g(x,u)-vh(u))\big)\\
&=\underset{x}{\min} \big(f(x)+\max_{\lambda\geq 0} \max_{u}\;\underset{v\geq 0}{\min}\;\lambda(g(x,u)-vh(u))\big)\;.
\end{align*}
We now show that
\begin{align*}
\mu=\underset{x}{\min}\;\underset{v\geq 0}{\min}\max_{\lambda\geq 0}\;\underset{u}{\max} \;{\cal L}(x,v,\lambda,u)\;,
\end{align*}
noting the switch in the sequence of $\min$-$\max$. It is sufficient to show that for any $x$, 
\begin{align*}
\gamma &:=\max_{\lambda\geq 0}\;\underset{u}{\max}\;\underset{v\geq 0}{\min}\ \lambda(g(x,u)-vh(u)) \;
\nonumber\\
&=\underset{v\geq 0}{\min}\max_{\lambda\geq 0}\;\underset{u}{\max} \; \lambda (g(x,u)-v h(u))\;.
\end{align*}
Let 
\begin{align}
{\cal G}(x):=\max_{h(u)\leq 0} g(x,u)=\underset{u}{\max}\;\underset{v\geq 0}{\min}\; (g(x,u)-vh(u))\;. \label{G_x}
\end{align}
So,
\begin{align}
\gamma=\max_{\lambda\geq 0}{\cal G}(x)=\left\{\begin{array}{ll} 0 & {\cal G}(x)\leq 0\\ \infty & {\cal G}(x)>0\end{array}\right.\;.  \label{gamma_2}
\end{align}
From strong duality for the parametric optimization problem (\ref{G_x}), we have 
\begin{align*}
{\cal G}(x)=\underset{v\geq 0}{\min}\;\underset{u}{\max} (g(x,u)-vh(u))\;. 
\end{align*}
Now, consider the second part in (\ref{alphas}), that is
\begin{align*}
\underset{v\geq 0}{\min}\max_{\lambda\geq 0}\;\underset{u}{\max} \; \lambda(g(x,u)-vh(u))\;.
\end{align*}
Starting from the first $\underset{u}{\max}$ at right, we get

\begin{align}
&\underset{u}{\max} \; \lambda (g(x,u)-vh(u))= \nonumber \\ \quad 
    &\left\{
    \begin{array}{ll}
    0 &\lambda=0\\
    \underset{u}{\max} \; \lambda (g(x,u)-vh(u)) &\lambda>0,\; v\geq 0
    \end{array}\right.
\end{align}

Then, consider $\displaystyle\max_{\lambda\geq 0}$ as
\begin{align*}
\max_{\lambda\geq 0} & \left\{\begin{array}{ll}0 &\lambda=0\\ \underset{u}{\max} \; \lambda (g(x,u)-vh(u)) & \lambda>0,\; v\geq 0 \end{array}\right.= \nonumber \\ &\left\{\begin{array}{ll}
\infty & v\geq 0,\; \underset{u}{\max}\;(g(x,u)-vh(u))>0\\
0& v\geq 0,\; \underset{u}{\max}\;(g(x,u)-vh(u))=0\\
0& v\geq 0,\; \underset{u}{\max}\;(g(x,u)-vh(u))<0
\end{array}\right.\;.
\end{align*}
Lastly, consider $\textstyle\underset{v\geq 0}{\min}$ as
\begin{align*}
\underset{v\geq 0}{\min} & \left\{\begin{array}{ll}
\infty & v\geq 0,\; \underset{u}{\max}\;(g(x,u)-vh(u))>0\\
0& v\geq 0,\; \underset{u}{\max}\;(g(x,u)-vh(u))=0\\
0& v\geq 0,\; \underset{u}{\max}\;(g(x,u)-vh(u))<0
\end{array}\right.= \nonumber\\ &\left\{\begin{array}{ll}\infty & \underset{v\geq 0}{\min}\;\underset{u}{\max}\;(g(x,u)-vh(u))>0\\
0 &\underset{v\geq 0}{\min} \; \underset{u}{\max}\;(g(x,u)-vh(u)\leq 0 \end{array}\right.\;,
\end{align*}
which is equal to $\gamma$ as claimed in (\ref{gamma_2}). Since minimizations (maximizations) can always be combined, the above result shows that 
\begin{align*}
\mu=\min_{x,v\geq 0}\max_{\lambda\geq 0,u} \;{\cal L}(x,v,\lambda,u)\;.
\end{align*}
Note that ${\cal L}$ is (jointly) convex in $(x,v)$, but it is not (jointly) concave in $(\lambda, u)$; although it is concave in each of these variables. We next show that notwithstanding this issue, the optimal solution to $\mu$ is a saddle point. Specifically, we show that
\begin{align}
\mu=\min_{x,v\geq 0}\;\max_{\lambda\geq 0,u} \;{\cal L}(x,v,\lambda,u)=\max_{\lambda\geq 0,u}\;\min_{x,v\geq 0} \;{\cal L}(x,v,\lambda,u)\;. \label{mu_last}
\end{align}
To show this, note that strong duality in the upper level parametric optimization problem in (\ref{RO_single}) implies 
\begin{align*}
\mu&=\underset{x}{\min}\max_{\lambda\geq 0}\;(f(x)+\lambda {\cal G}(x))=\max_{\lambda\geq 0}\;\underset{x}{\min}\;(f(x)+\lambda\; {\cal G}(x))\\
&=\max_{\lambda\geq 0}\;\underset{x}{\min}\;\underset{u}{\max}\;\underset{v\geq 0}{\min}\;(f(x)+\lambda(g(x,u)-vh(u)))\;,
\end{align*}
where the last equality comes from the definition of ${\cal G}(x)$ in (\ref{G_x}). To obtain the result in (\ref{mu_last}), we need to show that for any $\lambda\geq 0$\;,
\begin{align*}
\eta&=\underset{x}{\min}\;\underset{u}{\max}\;\underset{v\geq 0}{\min}\;(f(x)+\lambda(g(x,u)-vh(u)))\nonumber \\ &=\underset{u}{\max}\;\underset{x}{\min}\;\underset{v\geq 0}{\min}\;(f(x)+\lambda(g(x,u)-vh(u)))\;.
\end{align*}
Note that
\begin{align*}
\underset{v\geq 0}{\min}\;(-v\lambda h(u))=\left\{\begin{array}{ll} 0&\lambda  h(u)\leq 0\\ -\infty & \lambda h(u)>0\end{array}\right.\;.
\end{align*}
So, we have 
\begin{align*}
\underset{u}{\max} &\; (f(x)+\lambda g(x,u)+\left\{\begin{array}{ll} 0& \lambda h(u)\leq 0\\ -\infty & \lambda h(u)>0\end{array}\right.)\\ = &f(x)+\max_{\lambda h(u)\leq 0}\lambda g(x,u)\;.
\end{align*}
Thus,
\begin{align*}
\eta=\underset{x}{\min}\max_{\lambda h(u)\leq 0}\;(f(x)+\lambda g(x,u))\;.
\end{align*}
Since $g(x,u)$ is convex in $x$ and concave in $u$ as for Assumption \ref{assume_ssds2}, so $f(x)+\lambda g(x,u)$ has the same properties for $\lambda \geq 0$. It follows that the result does not change if we swap the order of the optimizations. Hence,
\begin{align*}
\eta&=\underset{x}{\min}\max_{\lambda h(u)\leq 0}\;(f(x)+\lambda g(x,u))\\ &=\max_{\lambda h(u)\leq 0}\;\underset{x}{\min}\;(f(x)+\lambda g(x,u))\\
&=\underset{u}{\max}\;\underset{x}{\min}\;\underset{v\geq 0}{\min} (f(x)+\lambda g(x,u)-vh(u))\;,
\end{align*}
which completes the proof of Theorem \ref{theorem_saddle}.
\end{proof}

Let $z^\star=(x^\star,\lambda^\star,u^\star,v^\star)$ be the saddle point for the Lagrangian (\ref{lags}). Using the result of Theorem \ref{theorem_saddle}, it follows that $z^\star$ enjoys the saddle point property, namely 
\begin{align}
{\cal L}(x^\star,\lambda,u,v^\star)\leq {\cal L}(x^\star,\lambda^\star,u^\star,v^\star)\leq {\cal L}(x,\lambda^\star,u^\star,v)\;.
\label{saddle}
\end{align}
From the above discussion on the development of Lagrangian function ${\cal L}$, it follows that ${\cal RO}$ problem can be viewed as two connected optimization problems. The lower level optimization problem (\ref{lower_ro}) parameterized by $x$ involving maximization over uncertain variables $u^{(i)}$ s and the upper level optimization problem involving minimization over the decision variable $x$. This insight can be used to define the Karush-Kuhn-Tucker (KKT) conditions for the ${\cal RO}$ problem as follows.
\begin{definition} 
Recalling that $x=(w,t)$, the KKT point $(x^\star,\lambda^\star,u^\star,v^\star)$ for the ${\cal RO}$ problem (\ref{ROs}) can be defined as follows
\begin{align}
\partial_x {\cal L}(x^\star,\lambda^\star,u^\star,v^\star)&=0\;,
\nonumber \\ \partial_{u^{(i)}} {\cal L}^{(i)}(x^\star,u^{(i)\star},v^{(i)\star})&=0\;, \label{kkt1}\\
\lambda^\star&\geq0, \nonumber \\
\lambda^\star \big(L(I^{(i)}+u^{(i)\star},y^{(i)},w^\star)-&\nonumber\\
t^\star -v^{(i)\star} \; h^{(i)}(u^{(i)\star})\big)&=0\;,\label{kkt3}\\
v^{(i)\star} &\geq 0,\nonumber \\ \; v^{(i)\star} \; h^{(i)}(u^{(i)\star})&=0\;,\label{kkt3_1}\\
L(I^{(i)}+u^{(i)\star},y^{(i)},w^\star)-&\nonumber\\
t^\star -v^{(i)\star} \; h^{(i)}(u^{(i)\star}) &\leq 0\; \nonumber,\\
h^{(i)}(u^{(i)\star}) &\leq 0\;,\label{kkt4}
\end{align}
for $i=1,\ldots,N$\;, where $\partial_x f $ is the notation for the gradient of $f$ w.r.t. $x$, ${\cal L}$ defined in (\ref{lags}), and ${\cal L}^{(i)}(x,u^{(i)},v^{(i)}):=L(I^{(i)}+u^{(i)},y^{(i)},w)
-v^{(i)} \; h^{(i)}(u^{(i)})$ for $i=1\ldots,N$.
\end{definition}

We now propose the following fundamental theorem on establishing the connection between the KKT and saddle point of the ${\cal RO}$ problem.
\begin{theorem} \label{KKT_saddle}
The KKT point $(x^\star,\lambda^\star,u^\star,v^\star)$ satisfying conditions (\ref{kkt1})-(\ref{kkt4}) also satisfies saddle point inequalities in (\ref{saddle}) and vice versa.
\end{theorem}
\begin{proof}
Considering the definitions in (\ref{fgh}), we first show that the KKT point satisfies the saddle point property. Note that
\begin{align*}
&{\cal L}(x^\star,\lambda^\star,u^\star,v^\star)-{\cal L}(x^\star,\lambda,u,v^\star)\\
=\;& \lambda^\star (g(x^\star,u^\star)-\textstyle \sum_{i=1}^N v^{(i)\star} h^{(i)}(u^{(i)\star}))\;- \\ & \lambda (g(x^\star,u)-\textstyle \sum_{i=1}^N v^{(i)\star} h^{(i)}(u^{(i)}))\\
=\;& (\lambda^\star+\lambda-\lambda) (g(x^\star,u^\star)-\textstyle \sum_{i=1}^N v^{(i)\star} h^{(i)}(u^{(i)\star}))\; -\\
&\lambda (g(x^\star,u)-\textstyle \sum_{i=1}^N v^{(i)\star} h^{(i)}(u^{(i)}))\\
=\;&\lambda (g(x^\star,u^\star)-\textstyle \sum_{i=1}^N v^{(i)\star} h^{(i)}(u^{(i)\star}))\;-\\
&\lambda (g(x^\star,u)-\textstyle \sum_{i=1}^N v^{(i)\star} h^{(i)}(u^{(i)}))\;+\\
&(\lambda^\star-\lambda)(g(x^\star,u^\star)-\textstyle \sum_{i=1}^N v^{(i)\star} h^{(i)}(u^{(i)\star}))\;.
\end{align*}

Since $u^\star$ is maximizing $ g(x^\star,u)-\textstyle \sum_{i=1}^N v^{(i)\star} h^{(i)}(u^{(i)})$, we have
\begin{align*}
(g(x^\star,u^\star)-\textstyle \sum_{i=1}^N v^{(i)\star} h^{(i)}(u^{(i)\star}))-&\\(g(x^\star,u)-\textstyle \sum_{i=1}^N v^{(i)\star} h^{(i)}(u^{(i)}))&\geq 0\;.
\end{align*}
By complimentary slackness property of the KKT point, we get $ \lambda^\star(g(x^\star,u^\star)-\textstyle \sum_{i=1}^N v^{(i)\star} h^{(i)}(u^{(i)\star}))=0$ and $g(x^\star,u^\star)\leq 0$. Combining all these implies
\begin{align*}
{\cal L}(x^\star,\lambda^\star,u^\star,v^\star)-{\cal L}(x^\star,\lambda,u,v^\star)\geq 0\;.
\end{align*}
We next show that ${\cal L}(x,\lambda^\star,u^\star,v)-{\cal L}(x^\star,\lambda^\star,u^\star,v^\star)$ is non-negative. Note that
\begin{align*}
&{\cal L}(x,\lambda^\star,u^\star,v)-{\cal L}(x^\star,\lambda^\star,u^\star,v^\star)\\
=& f(x)-f(x^\star)+\lambda^\star(g(x,u^\star)-\textstyle \sum_{i=1}^N  v^{(i)} h^{(i)}(u^{(i)\star}) \;-\\ &\lambda^\star(g(x^\star,u^\star)-\textstyle \sum_{i=1}^N v^{(i)\star} h^{(i)}(u^{(i)\star})))\\
=& f(x)-f(x^\star)+ \lambda^\star(g(x,u^\star)-g(x^\star,u^\star))\;+\\ &\sum_{i=1}^N\lambda^\star (v^{(i)^\star}-v^{(i)})h^{(i)}(u^{(i)\star})\;.
\end{align*}

Since $(x^\star,v^\star)$ minimizes the Lagrangian, we have 
\[f(x)-f(x^\star)+ \lambda^\star(g(x,u^\star)-g(x^\star,u^\star))\geq 0\;.\]
Similarly, using complimentary slackness condition and the fact that $h^{(i)}(u^{(i)\star})\leq 0$, $v^{(i)}\geq 0$, and $\lambda^\star\geq 0$, it follows that 
$ \textstyle \sum_{i=1}^N\lambda^\star (v^{(i)\star}-v^{(i)})h^{(i)}(u^{(i)\star})\geq 0\;.$\\
Now, we show that saddle point satisfies KKT conditions. Note that
\begin{align*}
&\min_{x,v} \;{\cal L}(x,\lambda^\star,u^\star,v)={\cal L}(x^\star,\lambda^\star,u^\star,v^\star)\leq {\cal L}(x,\lambda^\star,u^\star,v),\\
&\max_{u,\lambda} \;{\cal L}(x^\star,\lambda,u,v^\star)={\cal L}(x^\star,\lambda^\star,u^\star,v^\star)\geq {\cal L}(x^\star,\lambda,u,v^\star)\;.
\end{align*}
Hence,
\begin{align*}
\partial_xf(x^\star)+\lambda^\star \partial_{x}g(x^\star,u^\star)&=0\;,\\\partial_{u^{(i)}} g(x^\star,u)-\textstyle \sum_{i=1}^N v^{(i)\star} \partial_{u^{(i)}} h^{(i)}(u^{(i)\star})&=0\;.
\end{align*}
To show complimentary slackness, consider the optimization problem with fixed $x=x^\star$ as
\begin{align*}
\max_{u^{(i)},\forall i} \;\;g(x^\star,u)\nonumber\;
{\rm s.t.}\;\;h^{(i)}(u^{(i)})\leq 0\;,\;i=1,\ldots,N\;.
\end{align*}
With $g$ concave in $u$ and each $h^{(i)}$ convex in $u^{(i)}$, the above problem is convex with zero duality gap and hence, based on convex optimization theory \cite{boyd2004}, we have
\begin{align*}
g(x^\star,u^\star)=&\;G(x^\star,v^\star)\nonumber\\
=&\;\max_{u^{(i)},\forall i} \left(g(x^\star,u)-\textstyle \sum_{i=1}^N v^{(i)\star} h^{(i)}(u^{(i)})\right)\\
\geq& \; g(x^\star,u^\star)-\textstyle \sum_{i=1}^N v^{(i)\star} h^{(i)}(u^{(i)\star})\geq\; g(x^\star,u^\star)\;.
\end{align*}
The first inequality is true because $h^{(i)}(u^{(i)\star})\leq 0$ and $v^{(i)\star} \geq 0$. Hence, from the last inequality we get $v^{(i)\star} h^{(i)}(u^{(i)\star})=0$. We next show that $\lambda^\star g(x^\star,u^\star)=0$. For fixed $u^\star$, consider the optimization problem 
\begin{align*}
&\underset{x}{\min} \;\;f(x)\nonumber\;\;{\rm s.t.}\;\;g(x,u^\star)\leq 0\;.
\end{align*}
For fixed $u^\star$, above is a convex optimization problem and hence, we have zero duality gap. Then similarly,
\begin{align*}
f(x^\star)=&\;F(\lambda^\star, u^\star)=\;\inf_{x} f(x)+\lambda^\star g(x,u^\star)\nonumber\\
\leq &\; f(x^\star)+\lambda^\star g(x^\star,u^\star)
\leq \; f(x^\star)\;.
\end{align*}
The first inequality is true because $g(x^\star,u^\star)\leq 0$, and hence, the last inequality implies $\lambda^\star g(x^\star,u^\star)=0$. This completes the proof of Theorem \ref{KKT_saddle}.
\end{proof}

We can specify the equilibrium point $ (x^\star,\lambda^\star,u^\star,v^\star)$ of the dynamical system (\ref{sspd1s})-(\ref{sspd4s}) as
\begin{align*}
\partial_x f(x^\star)+\lambda^\star \partial_x g(x^\star,u^\star,\xi)&=0,\\
\partial_{u^{(i)}} g(x^\star,u^\star,\xi)-v^{(i)\star} \partial_{u^{(i)}} h^{(i)}(u^{(i)\star})&=0\;,\\
\lambda^\star g(x^\star,u^\star,\xi)&=0,\\
\lambda^\star &\geq 0,\\
g(x^\star,\lambda^\star,\xi) &\leq 0,\\
\lambda^\star v^{(i)\star} h^{(i)}(u^{(i)\star}) &=0,\\ v^{(i)\star}&\geq 0,\\
h^{(i)}(u^{(i)\star})&\leq 0.\label{kkt_ssds}
\end{align*}
for $i=1,\ldots,N$. The above conditions can also be viewed as the generalization of the KKT conditions from the deterministic setting to stochastic setting. Furthermore, by defining the Lagrangian function, ${\cal L}(x,\lambda,u,v,\xi):=f(x)+\lambda(g(x,u,\xi)-\textstyle \sum_{i=1}^N v^{(i)} h^{(i)}(u^{(i)}))$, following generalization of saddle point condition from deterministic setting (\ref{saddle}) to stochastic setting can be considered
\begin{align}
{\cal L}(x^\star,\lambda,u,v^\star,\xi)\leq {\cal L}(x^\star,\lambda^\star,u^\star,v^\star,\xi)\leq {\cal L}(x,\lambda^\star,u^\star,v,\xi)\;.
\end{align}

\subsection*{Convergence Proof of Stochastic Version of the Algorithm}

For the convergence proof of SSDS algorithm in (\ref{sspd1s})-(\ref{sspd4s}), we will assume $\lambda^\star>0$ and that numbers $R$ and $R_\lambda \leq R$ are known satisfying
\begin{align*}
\Vert z_{1}\Vert_{2} \leq R,\ \Vert z^\star\Vert_{2} \leq R,\ \Vert \lambda^\star\Vert_{2} \leq R_\lambda \ .
\end{align*}
We will also assume that the norm of the subgradients of $f$, $g$ and $h^{(i)}$ s, and the values of $f$, $g$ and $h^{(i)}$ s are bounded on compact sets based on Assumption \ref{assume_ssds2}. Let us start by defining the compact notations
\begin{align*}
{\cal N}(\Vert z_{k+1}-z^{\star}\Vert_{2}^{2}):=\;&\Vert x_{k+1}-x^{\star}\Vert_{2}^{2}+\Vert \lambda_{k+1}-\lambda^{\star}\Vert_{2}^{2}\\+&\lambda^\star \Vert u_{k+1}-u^{\star}\Vert_{2}^{2}+\Vert v_{k+1}-v^{\star}\Vert_{2}^{2} \ ,\\
(z_{k+1}-z^{\star})_{\lambda^\star}:=\;&(x_{k+1}-x^{\star})+( \lambda_{k+1}-\lambda^{\star})\\ +&\lambda^\star (u_{k+1}-u^{\star})+( v_{k+1}-v^{\star}) \ ,\\
\Vert T \Vert_{2,\lambda^\star}^2:=\;&\Vert T^{(x)} \Vert_2^2+\Vert T^{(\lambda)} \Vert_2^2 \\ +&\lambda^\star \Vert T^{(u)} \Vert_2^2+\Vert T^{(v)} \Vert_2^2 \ .
\end{align*}
By using the non-expansive property of positive projection operations for $\lambda$ and $v$ iterations, we write out the following basic equations

\begin{align*}
&\;E_{\xi_k}[{\cal N}(\Vert z_{k+1}-z^{\star}\Vert_{2}^{2})]\\
=&\;E_{\xi_k}[\Vert x_{k}-\alpha_{k} (\partial_x f(x_{k},\xi_k)+\lambda_k \partial_x g(x_k,u_k,\xi_k)-x^{\star}\Vert_{2}^{2}]\\
&+E_{\xi_k}[\Vert [\lambda_k+\alpha_k g(x_k,u_k,\xi_k)-v_k h(u_k)]_+-\lambda^{\star}\Vert_{2}^{2}]\\
&+E_{\xi_k}[\lambda^\star \Vert u_k+\alpha_k (\partial_u g(x_k,u_k,\xi_k)-v_k \partial_u h(u_k)-u^{\star}\Vert_{2}^{2}]\\
&+E_{\xi_k}[\Vert [v_k+\alpha_k (\lambda_k h(u_k))]_+-v^{\star}\Vert_{2}^{2}]\\
\leq &\;E_{\xi_k}[\Vert x_k-x^{\star}-\alpha_k (\partial_x f(x_k,\xi_k)+\lambda_k \partial_x g(x_k,u_k,\xi_k))\Vert_{2}^{2}]\\
&+E_{\xi_k}[\Vert \lambda_k-\lambda^{\star}+\alpha_k (g(x_k,u_k,\xi_k)-v_k h(u_k))\Vert_{2}^{2}]\\
&+E_{\xi_k}\lambda^\star \Vert u_k-u^{\star}+\alpha_k (\partial_u g(x_k,u_k,\xi_k)-v_k \partial_u h(u_k))\Vert_{2}^{2}]\\
&+E_{\xi_k}[\Vert v_k-v^{\star}+\alpha_k (\lambda_k h(u_k))\Vert_{2}^{2}]\\
=&\;\Vert x_k-x^{\star}\Vert_{2}^{2}+\Vert \lambda_k-\lambda^{\star}\Vert_{2}^{2}+\lambda^\star \Vert u_k-u^{\star}\Vert_{2}^{2}+\Vert v_k-v^{\star}\Vert_{2}^{2}\\
&-2E_{\xi_k}[\alpha_k(\partial_x f(x_k,\xi_k)+\lambda_k \partial_x g(x_k,u_k,\xi_k)) ^\top(x_k-x^{\star})]\\
&+2E_{\xi_k}[\alpha_k(g(x_k,u_k,\xi_k)-v_k h(u_k))^\top(\lambda_k-\lambda^{\star})]\\
&+2E_{\xi_k}[\alpha_k\lambda^\star (\partial_u g(x_k,u_k,\xi_k)-v_k \partial_u h(u_k)) ^\top(u_k-u^{\star})]\\
&+E_{\xi_k}[2\alpha_k(\lambda_k h(u_k))^\top(v_k-v^{\star})]\\
&+E_{\xi_k}[(\alpha_k)^{2}\Vert \partial_x f(x_k,\xi_k)+\lambda_k \partial_x g(x_k,u_k,\xi_k)\Vert_{2}^{2}]\\
&+E_{\xi_k}[(\alpha_k)^{2}\Vert g(x_k,u_k,\xi_k)-v_k h(u_k)\Vert_{2}^{2}]\\
&+E_{\xi_k}[(\alpha_k)^{2} \lambda^\star \Vert \partial_u g(x_k,u_k,\xi_k)-v_k \partial_u h(u_k) \Vert_{2}^{2}]\\
&+E_{\xi_k}[(\alpha_k)^{2}\Vert \lambda_k h(u_k) \Vert_{2}^{2}] \ .
\end{align*}

Using the compact notation, this reads to be
\begin{align*}
&E_{\xi_k}[{\cal N}(\Vert z_{k+1}-z^{\star}\Vert_{2}^{2})] \leq {\cal N}(\Vert z_k-z^{\star}\Vert_{2}^{2})\\&-2E_{\xi_k}[\alpha_k T_k^\top{\cal N}(z_{k+1}-z^{\star})]+E_{\xi_k}[\alpha_k^2 \; \Vert T_k \Vert_{2,\lambda^\star}^2] \;.
\end{align*}

Considering the upper bound $R_\lambda$ for two-norm of $\lambda^\star$, we can write
\begin{align*}
&E_{\xi_k}[{\cal N}(\Vert z_{k+1}-z^{\star}\Vert_{2}^{2})] \leq {\cal N}(\Vert z_k-z^{\star}\Vert_{2}^{2})\\& -2E_{\xi_k}[\alpha_k T_k^\top{\cal N}(z_{k+1}-z^{\star})]+C \; \gamma_k^2 \; .
\end{align*}
Taking expectation on both the sides with respect to $E_{\xi_0^{k-1}}$ on both the sides  and using the fact that $\xi_0^{k-1}$ is independent of $\xi_k$, we obtain
\begin{align*}
&E_{\xi_0^k}[{\cal N}(\Vert z_{k+1}-z^{\star}\Vert_{2}^{2})] \leq E_{\xi_0^{k-1}}{\cal N}(\Vert z_k-z^{\star}\Vert_{2}^{2})\\&-2E_{\xi_0^k}[\alpha_k T_k^\top{\cal N}(z_{k+1}-z^{\star})]+C \; \gamma_k^2 \;,
\end{align*}
\begin{align}
&E_{\xi_0^{k-1}}[{\cal N}(\Vert z_{k}-z^{\star}\Vert_{2}^{2})] \leq E_{\xi_0^{k-2}}{\cal N}(\Vert z_{k-1}-z^{\star}\Vert_{2}^{2})\\&-2E_{\xi_0^{k-1}}[\alpha_{k-1} T_{k-1}^\top{\cal N}(z_{k}-z^{\star})]+C \; \gamma_{k-1}^2,
\end{align}
where $C$ is defined as $\max\{1,R_\lambda\}$. 
Substituting  inequality (9) into (8)  we obtain
\begin{align*}
&E_{\xi_0^k}[{\cal N}(\Vert z_{k+1}-z^{\star}\Vert_{2}^{2})] \leq E_{\xi_0^{k-2}}{\cal N}(\Vert z_{k-1}-z^{\star}\Vert_{2}^{2})\\&-2E_{\xi_0^{k-1}}[\alpha_{k-1} T_{k-1}^\top{\cal N}(z_{k}-z^{\star})]\\&-2E_{\xi_0^k}[\alpha_k T_k^\top{\cal N}(z_{k+1}-z^{\star})]\\
&+C (\gamma_k^2+\gamma_{k-1}^2) \;.
\end{align*}
Using recursion, we obtain
\begin{align*}
&E_{\xi_0^k}[{\cal N}(\Vert z_{k+1}-z^{\star}\Vert_{2}^{2})] \leq E_{\xi_0}{\cal N}(\Vert z_{1}-z^{\star}\Vert_{2}^{2})\\&-2\sum_{i=1}^k E_{\xi_0^{i}}[\alpha^{(i)} T^{(i)\top}{\cal N}(z_{i+1}-z^{\star})]+C\sum_{i=1}^k{\gamma^{(i)}}^2 \;,
\end{align*}
\begin{align*}
&E_{\xi_0^k}[{\cal N}(\Vert z_{k+1}-z^{\star}\Vert_{2}^{2})]+2\sum_{i=1}^k E_{\xi_0^{i}}[\alpha^{(i)} T^{(i)\top}{\cal N}(z_{i+1}-z^{\star})] \\
\leq&\; E_{\xi_0}{\cal N}(\Vert z_{1}-z^{\star}\Vert_{2}^{2})+C\sum_{i=1}^k{\gamma^{(i)}}^2\leq C(4R^2+S) \;
\end{align*}

\begin{align}
&E_{\xi_0^k}[{\cal N}(\Vert z_{k+1}-z^{\star}\Vert_{2}^{2})]+2\Big( E_{\xi_0^{1}}[\alpha_{1} T_{1}^\top{\cal N}(z_{2}-z^{\star})]\nonumber\\&+E_{\xi_0^{2}}[\alpha_{2} T_{2}^\top{\cal N}(z_{3}-z^{\star})]+\ldots\nonumber\\
&+E_{\xi_0^{k}}[\alpha_{k} T_{k}^\top{\cal N}(z_{k+1}-z^{\star})]\Big) \leq E_{\xi_0}{\cal N}(\Vert z_{1}-z^{\star}\Vert_{2}^{2})\nonumber\\&+C\sum_{i=1}^k{\gamma^{(i)}}^2\leq C(4R^2+S) \;\label{compact_form11},
\end{align}

where the last inequality comes from the bounds on $\Vert z_{1}\Vert_{2}$, $\Vert z^\star\Vert_{2}$, $\Vert \lambda^\star\Vert_{2}$ and $\sum_{k=1}^{\infty}(\gamma_k)^{2}$.

We argue that the sum on the left-hand side of (\ref{compact_form11}) is non-negative.

\begin{align*}
E_{\xi_0^k}[\alpha_k T_k^\top {\cal N}(z_{k+1}-z^{\star})] =E_{\xi_0^{k-1}}[ [E_{\xi_k}[\alpha_k T_k^\top {\cal N}(z_{k+1}-z^\star)]]] 
\end{align*}
Where we have use the fact that $\xi_0^{k-1}$ is independent of $\xi_k$.
\begin{align*}
&E_{\xi_k}[ \alpha_k T_k^\top {\cal N}(z_{k+1}-z^{\star})] \\
=&\;E_{\xi_k}[ \alpha_k \partial_x f(x_k,\xi_k)+\alpha_k \lambda_k \partial_x g(x_k,u_k,\xi_k) ^\top(x_k-x^{\star})] \\
&-E_{\xi_k}[\alpha_k(g(x_k,u_k,\xi_k)-v_k h(u_k))^\top(\lambda_k-\lambda^{\star})] \\
&-E_{\xi_k}[\alpha_k\lambda^\star (\partial_u g(x_k,u_k,\xi_k)-v_k \partial_u h(u_k)) ^\top(u_k-u^{\star})]\\
&+E_{\xi_k}[-\alpha_k \lambda_k h(u_k)^\top(v_k-v^{\star})]\\
\geq &\; E_{\xi_k}[\alpha_k(f(x_k,\xi_k)-f(x^\star,\xi_k))]+E_{\xi_k}[\Ccancel[red]{\alpha_k \lambda_k g(x_k,u_k,\xi_k)}\\&-\alpha_k\lambda_k g(x^\star,u_k,\xi_k)] -E_{\xi_k}[\Ccancel[red]{\alpha_k \lambda_k g(x_k,u_k,\xi_k)}\\&+\Ccancel[blue]{\alpha_k\lambda^\star g(x_k,u_k,\xi_k)}]+E_{\xi_k}[\Ccancel[green]{\alpha_k \lambda_k v_k h(u_k)}-\Ccancel[magenta]{\alpha_k\lambda^\star v_k h(u_k)}]\\
&+E_{\xi_k}[\alpha_k \lambda^\star g(x_k,u^\star,\xi_k)-\Ccancel[blue]{\alpha_k\lambda^\star g(x_k,u_k,\xi_k)}]\\&+E_{\xi_k}[\Ccancel[magenta]{\alpha_k \lambda^\star v_k h(u_k)}] -E_{\xi_k}[\alpha_k\lambda^\star v_k h(u^\star)-\Ccancel[green]{\alpha_k \lambda_k v_k h(u_k)}\\&+\alpha_k \lambda_k v^\star h(u_k)]\\
=&\;E_{\xi_k}[\alpha_k(\big(f(x_k,\xi_k)+\lambda^\star g(x_k,u^\star,\xi_k)-\lambda^\star v_k h(u^\star)\big)]\\
&-E_{\xi_k}[\alpha_k\big(f(x^\star,\xi_k)+\lambda_k g(x^\star,u_k,\xi_k)-\lambda_k v^\star h(u_k)\big)]\\
=&\;E_{\xi_k}[\alpha_k{\mathcal L}(x_k,\lambda^\star,u^\star,v_k,\xi_k)]-E_{\xi_k}[\alpha_k{\mathcal L}(x^\star,\lambda_k,u_k,v^\star,\xi_k)]\\
\geq & \;E_{\xi_k}[ \alpha_k {\mathcal L}(x_k,\lambda^\star,u^\star,v_k,\xi_k)-\alpha_k{\mathcal L}(x^\star,\lambda^\star,u^\star,v^\star,\xi_k)] \geq 0 \ .
\end{align*}

Since $\alpha_k\geq 0$, we have from above that
\[E_{\xi_0^k}[\alpha_k T_k^\top {\cal N}(z_{k+1}-z^{\star})]\geq 0\]

\begin{remark} \label{strict_convexity_concavity}
Since $f$ is assumed to be strictly convex in $x$ and $g$  is strictly concave in $u$, 

the above inequality

is strict whenever  $x\neq x^\star$ and $u \neq u^\star$. Moreover, if the inequality   becomes an equality, we get $x=x^\star$ and  $u=u^\star$.
\end{remark}

We have
\begin{align*}
&E_{\xi_0^k}[{\cal N}(\Vert z_{k+1}-z^{\star}\Vert_{2}^{2})]\leq C (4R^{2}+S),\\
&2\sum_{i=1}^{k} \gamma^{(i)} E_{\xi_0^i}\left[ \frac{T^{(i)\top}}{\Vert T^{(i)} \Vert} (z^{(i)}-z^{\star})_{\lambda^\star}\right]
\leq C(4R^{2}+S)
\end{align*}
By assumption, the norm of Subgradients on the set $\Vert z_k\Vert_{2,\lambda^\star}\leq D$ is bounded, so it follows that $\Vert T_k\Vert_{2}$ is bounded. Because the sum of $\gamma_k$ diverges, for the sum
\begin{align*}
\sum_{i=1}^{k} \gamma^{(i)} E_{\xi_0^i}\left[\frac{T^{(i)^\top}}{\Vert T^{(i)} \Vert} (z_{i+1}-z^{\star})_{\lambda^\star}\right]
\end{align*}
to be bounded, we need
\begin{align*}
\lim_{k\rightarrow\infty}E_{\xi_0^k}\left[\frac{T_k^\top}{\Vert T_k \Vert_{2}}{\cal N}(z_{k+1}-z^{\star})\right]=0.
\end{align*}
Since $\Vert T_k\Vert_{2}$ is bounded, the numerator $E_{\xi_0^k}[T_k^\top{\cal N}(z_{k+1}-z^{\star})]$ has to go to zero in the limit. From Remark \ref{strict_convexity_concavity}, we conclude that
\begin{align*}
\lim_{k\rightarrow\infty} E_{\xi_0^{k-1}}[x_k]=x^\star,\;\;\;
\lim_{k\rightarrow\infty} E_{\xi_0^{k-1}}[u_k]=u^\star.
\end{align*}

\end{document}